\documentclass[a4paper]{article}

\usepackage{hyperref}       
\usepackage{url}            
\usepackage{booktabs}       
\usepackage{amsfonts}       
\usepackage{nicefrac}       
\usepackage{microtype}      
\usepackage{natbib}

\usepackage{amsmath,amssymb}
\usepackage{dsfont}
\usepackage{bbm}
\usepackage{graphicx}
\usepackage{color}
\usepackage{bbm}
\usepackage{multirow}
\usepackage{slashbox}
\usepackage{url}
\usepackage{balance}
\usepackage{multibib}
\usepackage{multicol}
\usepackage{mathtools}
\usepackage{arydshln}

\usepackage{algorithm}
\usepackage[noend]{algpseudocode}
\usepackage{subcaption}
\allowdisplaybreaks

\newcommand{\MNIST}[1]{\texttt{MNIST}$_{#1}$}
\newcommand{\Reuters}[1]{\texttt{Reuters}$_{#1}$}
\newcommand{\hyperprior}{\pi}
\newcommand{\hyperposterior}{\rho}

\newcommand{\prior}{P}
\newcommand{\posterior}{Q}

\newcommand{\KL}{\operatorname{KL}}
\newcommand{\kl}{\operatorname{kl}}

\newcommand{\bqmv}{B_\hyperposterior}

\newcommand{\Ind}[1]{\mathds{1}_{#1}}

\newcommand{\dmv}{d_{\Dcal}(\rho)}

\newcommand{\xbf}{{\bf x}}

\newcommand{\D}{{\cal D}}
\newcommand{\Dcal}{{\cal D}}
\newcommand{\Xcal}{{\cal X}}
\newcommand{\Ycal}{{\cal Y}}

\newcommand{\Hcal}{{\cal H}}  
\newcommand{\Vcal}{{\cal V}}  

\newcommand{\concat}{\texttt{Concat}}
\newcommand{\mono}{\texttt{Mono}}
\newcommand{\mvmv}{\texttt{MV-MV}}

\newcommand{\Fusion}[2]{\texttt{Fusion}$_{\mathtt{#1}} \pwr {\mathtt{#2}}$}
\newcommand{\mvboost}{\texttt{MVBoost}}
\newcommand{\mvAdaboost}{\texttt{MV-AdaBoost}}
\newcommand{\rboost}{\texttt{rBoost.SH}}
\newcommand{\pbboost}{\texttt{PB-MVBoost}}

\DeclareMathOperator*{\Esp}{\mathbb{E}} 
\newcommand{\E}[1]{{\displaystyle \Esp_{#1}}}

\newcommand{\sign}{\operatorname{sign}}
\newcommand{\pwr}{^}

\DeclareMathOperator*{\argmin}{\mathrm{argmin}}

\newtheorem{cor}{Corollary}
\newtheorem{lemma}{Lemma}
\newtheorem{theorem}{Theorem}
\newtheorem{definition}{Definition}

\usepackage{fancyhdr}
\usepackage{fancybox}
\newcommand{\ver}{1}
\title{Multiview Boosting by Controlling the Diversity and the Accuracy of View-specific Voters} 
\author{Anil Goyal$^{1,2}$ \and Emilie Morvant$^1$ \and Pascal Germain$^3$ \and  Massih-Reza Amini$^2$ \and 
	$\mbox{}^1$ \small Univ Lyon, UJM-Saint-Etienne, CNRS, Institut d'Optique Graduate School, \\ \small  Laboratoire Hubert Curien UMR 5516, F-42023, Saint-Etienne, France
	\and
	$\mbox{}^2$ \small Univ. Grenoble Alps, Laboratoire d'Informatique de Grenoble, AMA, \\ \small Centre Equation 4, BP 53, F-38041 Grenoble Cedex 9, France
	\and
	$\mbox{}^3$ \small  Inria Lille - Nord Europe, Modal Project-Team, \\ \small 59650 Villeneuve d'Ascq, France
}

\begin{document}

\maketitle

\begin{abstract}
In this paper we propose a boosting based multiview learning algorithm, referred to as $\pbboost$, which iteratively learns \textit{i)}  weights over view-specific voters capturing view-specific information; and  \textit{ii)}  weights over views by optimizing a PAC-Bayes multiview C-Bound that takes into account the accuracy of view-specific classifiers and the diversity between the views.
We derive a generalization bound for this strategy following the PAC-Bayes theory which is a suitable tool to deal with models expressed as weighted combination over a set of voters.
Different experiments on three publicly available datasets show the efficiency of the proposed approach with respect to state-of-art models.
\end{abstract}


\section{Introduction}

With the tremendous generation of data, there are more and more situations where observations are described by more than one view.
This is for example the case with multilingual documents that convey the same information in different languages or images that are naturally described according to different set of features (for example SIFT, HOG, CNN etc).
In this paper, we study the related machine learning problem that consists in finding an efficient classification model from different information sources that describe the observations.
This topic, called multiview learning~\cite{AtreyHEK10,sun2013survey}, has been expanding over the past decade, spurred by the seminal work of Blum and Mitchell on co-training~\cite{blum98} (with only two views).
The aim is to learn a classifier which performs better than classifiers trained over each view separately (called view-specific classifier). 
Usually, this is done by directly concatenating the 
representations (early fusion) or by combining the predictions of view-specific classifiers (late fusion)~\cite{Early-Late-ACMMultimedia05}.
In this work, we stand in the latter situation.
Concretely, we study a two-level multiview learning strategy based on the PAC-Bayesian theory 
(introduced by ~\cite{McAllester99} for monoview learning). 
This theory provides Probably Approximately Correct (PAC) generalization guarantees for models expressed as a weighted combination over a set of functions/voters ({\it i.e.}, for a weighted majority vote). 
In this framework, given a \textit{prior} distribution over a set of functions (called voters) $\cal H$ and a learning sample, one aims at learning a \textit{posterior} distribution over $\cal H$ leading to a well-performing majority vote where each voter from $\cal H$ is weighted by its probability to appear according to the posterior distribution.
Note that, PAC-Bayesian studies have not only been conducted to characterize the error of such weighted majority votes~\cite{catoni2007pac,Seeger02,LangfordS02,GermainLLMR15}, but have also been used to derive theoretically grounded learning algorithms ({\it e.g.} for supervised learning~\cite{GermainLLM09,parrado-12,alquier-15,cqboost,MorvantHA14} or transfer learning~\cite{DALC}).
\begin{figure}[t!]
	\centering
	\includegraphics[scale=0.4]{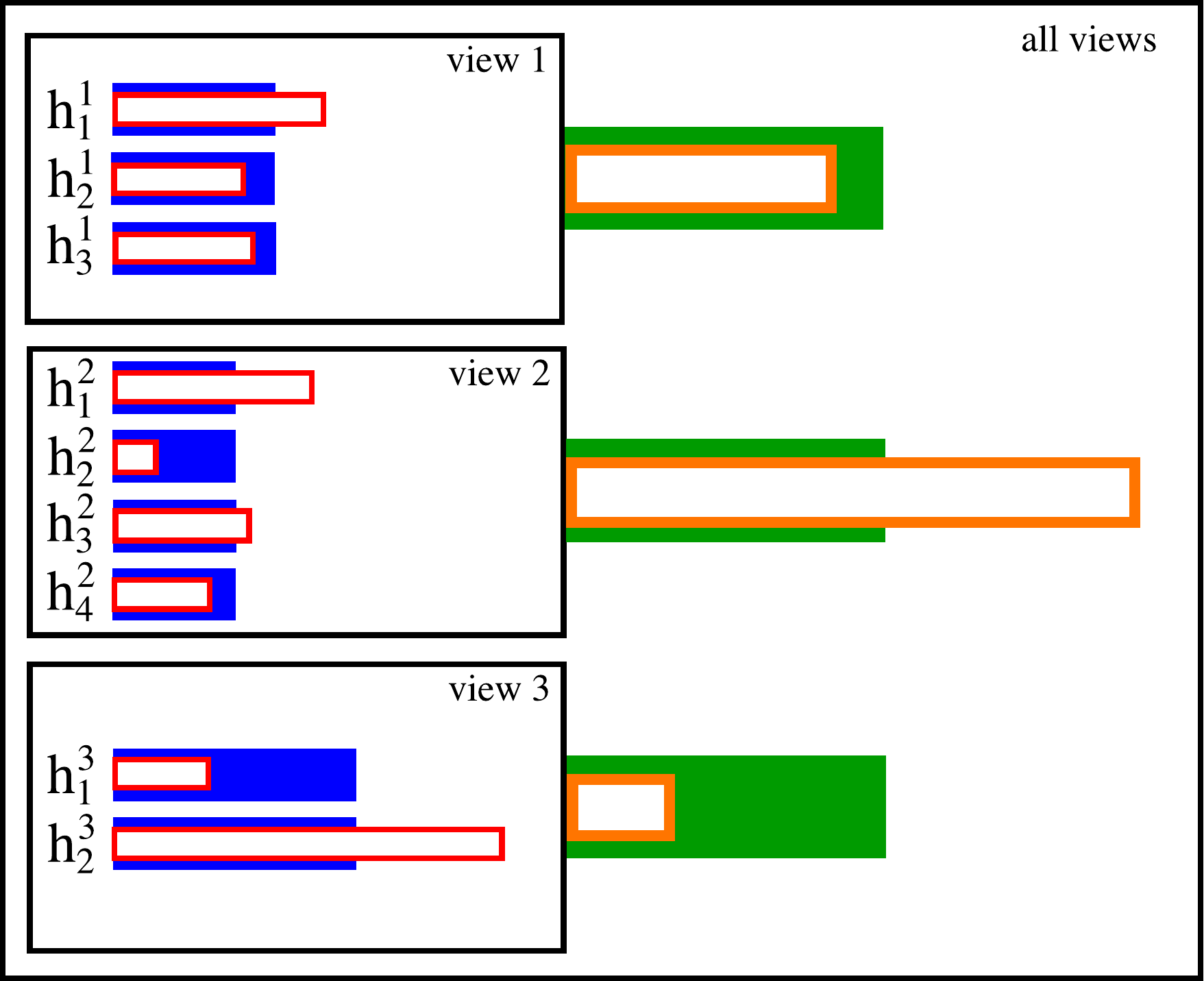}
	\caption{Example of the multiview distributions hierarchy with $3$ views. For all views $v\in\{1,2,3\}$, we have a set of voters  $\Hcal_v=\{h_1^v,\ldots,h_{n_v}^v\}$ on which we consider prior $\prior_v$ view-specific distribution (in blue), and  we consider a hyper-prior~$\hyperprior$ distribution (in green) over the set of $3$ views. The objective is to learn a posterior $\posterior_v$ (in red) view-specific distributions and a hyper-posterior $\hyperposterior$ distribution (in orange) leading to a good model. The length of a rectangle  represents the weight (or probability) assigned to a voter or a view.}
	\label{fig:MultiviewHierarchy}
\end{figure}
To tackle multiview learning in a PAC-Bayesian fashion, we propose to define a two-level hierarchy of  {prior} and {posterior} distributions over the views:
\textit{i)} for each view $v$, we consider a prior $P_v$ and a posterior $Q_v$  distributions over view-specific voters to capture view-specific information and \textit{ii)} a hyper-prior $\pi_v$ and  a hyper-posterior $\rho_v$  distributions over the set of views to capture the accuracy of view-specific classifiers and diversity between the views (see Figure~\ref{fig:MultiviewHierarchy}).
Following this distributions' hierarchy, we define a multiview majority vote classifier where view-specific classifiers are weighted according to posterior and hyper-posterior distributions.
By doing so, we extend the classical PAC-Bayesian theory to multiview learning with more than two views and derive a  PAC-Bayesian  generalization bound for our multiview majority vote classifier.

From a practical point of view, we design an algorithm based on the idea of boosting~\cite{Freund95,adaboost,Schapire99, Schapire2003}.
Our boosting-based multiview learning algorithm, called $\pbboost$, deals with the two-level hierarchical learning strategy.
$\pbboost$ is an ensemble method and outputs a multiview classifier that is a combination of view-specific voters. 
It is well known that controlling the diversity between the view-specific classifiers or the views is a key element in multiview learning \cite{Massih09, Goyal2017, chapelle2006semi, Kuncheva, Maillard09, MorvantHA14}.
Therefore, to learn the weights over the views, we minimize an upper-bound on the error of the majority vote, called the multiview C-bound~\cite{GermainLLMR15,cqboost,Goyal2017}, allowing us to control a trade-off between accuracy and diversity.
Concretely, at each iteration of our multiview algorithm, we learn \textit{i)} weights over view-specific voters based on their ability to deal with examples on the corresponding view (capturing view-specific informations); and \textit{ii)} weights over views by minimizing the multiview C-bound.
To show the potential of our algorithm, we empirically evaluate our approach on \MNIST{1}, \MNIST{2} and \Reuters{} RCV1/RCV2 collections\cite{Lecun98, Massih09}. 
We observe that our algorithm $\pbboost$, empirically minimizes the multiview C-Bound over iterations, and lead to good performances even when the classes are unbalanced. We compare $\pbboost$ with a previously developed multiview algorithm, denoted by \Fusion{all}{Cq} \cite{Goyal2017},  which first learns the view-specific voters at the base level of the hierarchy, and then, combines the predictions of view-specific voters using a PAC-Bayesian algorithm \texttt{CqBoost} \cite{cqboost}. From the experimental results, it came out that $\pbboost$ is more stable across different datasets and computationally faster than \Fusion{all}{Cq}.

In the next section, we discuss some related works. In Section~\ref{sec:MVPB-Bound}, we present the PAC-Bayesian  multiview learning framework~\cite{Goyal2017}. 
In Section~\ref{sec:pbboost}, we derive our multiview learning algorithm $\pbboost$. 
Before concluding in Section~\ref{sec:conclusion}, we experiment our algorithm in Section~\ref{sec:results}.

\section{Related Work}
\label{sec:related_work}

Learning a weighted majority vote is closely related to ensemble methods~\cite{Dietterich00,re2012ensemble}.
In the ensemble methods literature, it is well known that we desire to combine voters that make errors on different data points~\cite{Kuncheva}.
Intuitively, this means that the voters disagree on some data points.
This notion of disagreement (or agreement) is sometimes called diversity between classifiers~\cite{Pastor15,brown10,Kuncheva}.
Even if there is no consensus on the definition of ``diversity'', controlling it while keeping good accuracy is at the heart of a majority of ensemble methods: indeed if all the voters agree on all the points then there is no interest to combine them, only one will be sufficient.
Similarly, when we combine multiple views (or representations), it is known that controlling diversity between the views plays a vital role for learning the final majority vote~\cite{Massih09, Goyal2017, chapelle2006semi, Maillard09}.
Most of the existing ensemble-based multiview learning algorithms try to exploit either view consistency (agreement between views)~\cite{janodet:hal-00403242,Koco11,Xiao12} or diversity between views~\cite{emvboost,Goyal2017, Peng11,Peng17} in different manners. \cite{janodet:hal-00403242} proposed a boosting based multiview learning algorithm for 2 views, called 2-Boost.
At each iteration, the algorithm learns the weights over the view-specific voters by maintaining a single distribution over the learning examples.
Conversely, \cite{Koco11} proposed Mumbo that maintains separate distributions for each view.
For each view, the algorithm  reduces the weights associated with the examples hard to classify, and increases the weights of those examples in the other views.
This trick allows a communication between the views with the objective to maintain view consistency.
Compared to our approach, we follow a two-level learning strategy where we learn (hyper-)posterior distributions/weights over view-specific voters and views.
In order to take into account accuracy and diversity between the views, we optimize the multiview C-Bound (an upper-bound over the risk of multiview majority vote learned, see e.g. ~\cite{GermainLLMR15,cqboost,Goyal2017}).

\cite{emvboost} proposed EMV-AdaBoost, an embedded multiview Adaboost algorithm, restricted to two views.
At each iteration, an example contributes to the error if it is misclassified by any of the view-specific voters and the diversity between the views is captured by weighting the error by the agreement between the views.
\cite{Peng11,Peng17} proposed variants of Boost.SH (boosting with SHared weight distribution) which controls the diversity for more than two views. 
Similarly than our approach, they maintain a single global distribution over the learning examples for all the views.
To control the diversity between the views, at each iteration they update the distribution over the views by casting the algorithm in two ways: \textit{i)} a multiarmed bandit framework ($\texttt{rBoost.SH}$) and  \textit{ii)} an expert strategy framework ($\texttt{eBoost.SH}$) consisting of set of strategies (distribution over views) for weighing views. 
At the end, their multiview majority vote is a combination of $T$ weighted base voters, where $T$ is the number of iterations for boosting. 
Whereas, our multiview majority vote is a weighted combination of the view-specific voters over all the weighted views. 

Furthermore, our approach encompasses the one of \cite{Massih09} and \cite{Xiao12}. 
\cite{Massih09} proposed a Rademacher analysis for expectation of individual risks of each view-specific classifier (for more than two views). 
\cite{Xiao12} derived a weighted majority voting Adaboost algorithm which learns weights over view-specific voters at each iteration of the algorithm.
Both of these approaches maintain a uniform distribution over the views whereas our algorithm learns the weights over the views such that they capture diversity between the views.
Moreover, it is important to note that \cite{SunS14} proposed a PAC-Bayesian analysis for multiview learning over the concatenation of views but limited to two views and to a particular kind of voters: linear classifiers.
This has allowed them to derive a SVM-like learning algorithm but dedicated to multiview with exactly two views.
In our work, we are interested in learning from more than two views and no restriction on the classifier type.
Contrary to them, we followed a two-level distributions' hierarchy where we learn weights over view-specific classifiers and weights over views.

\section[Multiview PB Framework]{The Multiview PAC-Bayesian Framework} 
\label{sec:MVPB-Bound}

\subsection{Notations and Setting}
In this work, we tackle multiview binary classification tasks where the observations are described with $V\geq 2$ different representation spaces, {\it i.e.}, views; let $\Vcal$ be the set of these $V$ views.  
Formally, we focus on tasks for which the input space is  $\Xcal = \Xcal_1 \times \dots \times \Xcal_V$, where $\forall v\in\Vcal,\ X_v \subseteq \mathbb{R}^{d_v}$ is a \mbox{$d_v$-dimensional} input space, and the binary output space is $\Ycal =\{-1,+1\}$.
We assume that $\Dcal$ is an unknown distribution over $\Xcal\times \Ycal$.
We stand in the PAC-Bayesian supervised learning setting where an observation $\mathbf{x} = (x \pwr 1,x \pwr 2, \dots, x \pwr V)\in\Xcal$ is given with its label $y\in\Ycal$, and is independently and identically drawn ({\it i.i.d.}) from $\Dcal$.
A learning algorithm is then provided with a training sample $S$ of $n$ examples {\it i.i.d.} from $\Dcal$: $S = \{(\mathbf{x}_i, y_i)\}_{i=1} \pwr n \sim (\Dcal)^n$, where $(\Dcal)^n$ stands for the distribution of a \mbox{$n$-sample}.
For each view $v \in \Vcal$, we consider a view-specific set $\Hcal_v$ of voters $h  :  \mathcal{X}_v \to {\cal Y}$, and a  prior distribution $\prior_v$ on $\Hcal_v$. Given a \emph{hyper-prior} distribution $\hyperprior$ over the views $\Vcal$, and a multiview learning sample $S$, our PAC-Bayesian learner \mbox{objective} is twofold: {\it i)} finding a posterior distribution $\posterior_v$ over $\mathcal{H}_v$  for all views $v \in \Vcal$; {\it ii)} finding a \emph{hyper-posterior} distribution $\hyperposterior$ on the set of the views $\Vcal$. 
This defines a hierarchy of distributions illustrated on Figure~\ref{fig:MultiviewHierarchy}. 
The learned distributions express a multiview weighted majority vote\footnote{In the PAC-Bayesian literature, the weighted majority vote is sometimes called the Bayes classifier.} defined as 
\begin{align}
\label{eq:mv_majority_vote}
B_{\hyperposterior}(\mathbf{x}) = \sign \left[\E{ v \sim \hyperposterior}\  \E{ h \sim \posterior_v} h(x^v) \right].
\end{align}
Thus, the learner aims at constructing the posterior and hyper-posterior distributions that minimize the true risk $R_{\mathcal{D}}(B_{\hyperposterior})$ of the multiview weighted majority vote
$$ R_{\mathcal{D}}(B_{\hyperposterior}) = \E{(\mathbf{x},y) \sim \mathcal{D}} \Ind{[B_{\hyperposterior}(\mathbf{x}) \neq y]},
$$ 
where $\Ind{[\pi]} = 1$ if the predicate $\pi$ is true and $0$ otherwise. 
The above risk of the deterministic weighted majority vote is closely related to the Gibbs risk $R_{\Dcal} (G_{\hyperposterior})$ defined as the expectation of the individual risks of each voter that appears in the majority vote. More formally, in our multiview setting, we have
\begin{align}
\nonumber 
 R_{\mathcal{D}}(G_{\hyperposterior})  \ =\  \E{(\xbf,y) \sim \D} \ \E{ v \sim \hyperposterior} \ \E{h \sim \posterior_v} \Ind{[h(x^v) \neq y]},
\end{align}
and its empirical counterpart is 
\begin{align*}
R_{S}(G_{\hyperposterior}) &= \frac{1}{n} \sum_{i=1}^n \E{ v \sim \hyperposterior} \  \E{ h \sim \posterior_v} \Ind{[h(x_i^v) \neq y_i]}.
\end{align*}
In fact, if $B_{\hyperposterior}$ misclassifies $\xbf \in \Xcal$, then at least half of the view-specific voters from all the views (according to  hyper-posterior and posterior distributions) makes an error on $\xbf$. Then, it is well known ({\it e.g.}, \cite{shawe2003pac,McAllester03,GermainLLMR15}) that $R_{\Dcal} (B_{\rho})$ is upper-bounded by twice $R_{\Dcal} (G_{\rho})$:
$$
R_{\mathcal{D}}(B_\rho)\leq 2R_{\mathcal{D}}(G_{\rho}).
$$
In consequence, a generalization bound for $R_{\mathcal{D}}(G_{\hyperposterior})$ gives rise to a generalization bound for $R_{\mathcal{D}}(B_{\hyperposterior})$.

There exist other tighter relations~\cite{LangfordS02,Lacasse06,GermainLLMR15}, such as the C-Bound~\cite{Lacasse06,GermainLLMR15} which captures a trade-off between the Gibbs risk $R_{\mathcal{D}}(G_{\rho})$ and the disagreement between pairs of voters. This latter can be seen as a measure of diversity among the voters involved in the majority vote~\cite{MinCQ,MorvantHA14}, that is a key element to control from a multiview point of view~\cite{Massih09,AtreyHEK10,Kuncheva,Maillard09,Goyal2017}. 
The C-Bound can be extended to our multiview setting as below. 
\begin{lemma}[Multiview C-Bound] \label{lem:mv-cbound}
	Let $V \geq 2$ be the number of views. 
	For all posterior $\{\posterior_v\}_{v=1}^V$ distributions over $\{\Hcal_v\}_{v=1}^V$ and hyper-posterior $\hyperposterior$ distribution over views $\mathcal{V}$, if $R_{\Dcal}(G_{\hyperposterior}) <\frac12$, then we have
\begin{eqnarray}
\label{eq:cbound_multiview}
R_{\mathcal{D}}(B_\hyperposterior)
&\leq &
1- \frac{\displaystyle  \big(1-2R_{\Dcal}(G_{\hyperposterior})\big)^2}{\displaystyle 1-2 \dmv} \label{eq:cbound_multiview_4}\\
&\leq & 
1- \frac{  \Big(1-2  \Esp_ {v \sim \hyperposterior} R_{\Dcal}(G_{Q_v})\Big)^2}{ 1-2 \Esp_ {v \sim \hyperposterior}  d_{\Dcal}(Q_v) } \label{eq:cbound_multiview_2}\,,
\end{eqnarray}
where $\dmv$  is the expected disagreement between pairs of voters defined as
\begin{align*}
\dmv \ & = \  \Esp_{\xbf \sim \mathcal{D}_{\mathcal{X}}} \Esp_ {v \sim \hyperposterior} \Esp_{v' \sim \hyperposterior}   \Esp_{h \sim \posterior_v}  \Esp_{h' \sim \posterior_{v'}}   \Ind{[ h(x^v) {\ne} h'(x^{v'})]},
\end{align*}
and $R_\Dcal (G_{Q_v})$ and  $d_{\Dcal}(Q_v)$  are respectively the true view-specific Gibbs risk and the expected disagreement defined as
\begin{align*}
R_{\mathcal{D}}(G_{Q_v})  &=  \E{(\xbf,y) \sim \D} \  \E{h \sim \posterior_v} \Ind{[h(x^v) \neq y]}\,, \\
d_{\Dcal}(Q_v) & =  \Esp_{\xbf \sim \mathcal{D}_{\mathcal{X}}}   \Esp_{h \sim \posterior_v}  \Esp_{h' \sim \posterior_{v}}   \Ind{[ h(x^v) {\ne} h'(x^{v})]}.
\end{align*}

\end{lemma}
\begin{proof} 
  Similarly than done for the classical C-Bound~\cite{Lacasse06,GermainLLMR15}, Equation~\eqref{eq:cbound_multiview} follows from the Cantelli-Chebyshev's inequality (we provide the proof in \ref{sec:proof_cbound}). \\
Equation~\eqref{eq:cbound_multiview_2} is obtained by rewriting $R_{\mathcal{D}}(G_{\hyperposterior})$ as the $\hyperposterior$-average of the risk associated to each view, and the lower-bounding $\dmv$  by the $\hyperposterior$-average of the disagreement associated to each view.
First we notice that in the binary setting where $y\in\{-1,1\}$ and $h:\Xcal\to\{-1,1\}$, we have 
$\Ind{[h(x^v) \neq y]}  =  \frac12 (1-y\,h(x^v))$, and
\begin{align}
R_{\mathcal{D}}(G_{\rho}) 
  \nonumber 
\ &=\ \E{(\mathbf{x},y) \sim \mathcal{D}} \ \E{ v \sim \hyperposterior} \ \E{h \sim \posterior_v} \Ind{[h(x^v) \neq y]} \nonumber  \\
& 
 =\ \frac{1}{2}\bigg(1- \E{(\xbf,y) \sim \mathcal{D}} \  \E{v \sim \hyperposterior} \ \E{h \sim \posterior_v} y\,h(x^v) \bigg) \nonumber  \\
  & 
  =\  
 \E{v \sim \hyperposterior} R_{\mathcal{D}}(G_{Q_v})\,.\nonumber 
\end{align}
Moreover, we have
\begin{align}
\dmv 
 \nonumber
\ &=\  \Esp_{\xbf \sim \mathcal{D}_{\mathcal{X}}} \Esp_ {v \sim \hyperposterior} \Esp_{v' \sim \hyperposterior}   \Esp_{h \sim \posterior_v}  \Esp_{h' \sim \posterior_{v'}} \Ind{[ h(x^v) {\ne} h'(x^{v'})]} \\
&=\ \frac{1}{2} \bigg( 1- \Esp_{\xbf \sim \mathcal{D}_{\mathcal{X}}} \Esp_ {v \sim \hyperposterior} \Esp_{v' \sim \hyperposterior}   \Esp_{h \sim \posterior_v}  \Esp_{h \sim \posterior_{v'}}  h(x^v) \times h'(x^{v'})  \bigg) \nonumber \\
\nonumber&=\ \frac{1}{2} \bigg( 1- \E{\xbf \sim \mathcal{D}_{\mathcal{X}}}\bigg[  \E{v \sim \hyperposterior} \, \E{h \sim \posterior_v} h(x^v) \bigg]^2  \bigg)  \,.
\end{align}
 From Jensen's inequality (Theorem~\ref{theo:jensen}, in Appendix) it comes
\begin{align*}
\dmv\ 
&\geq\ \frac{1}{2} \bigg( 1- \E{\xbf \sim \mathcal{D}_{\mathcal{X}}} \,  \E{v \sim \hyperposterior}\bigg[  \E{h \sim \posterior_v} h(x^v) \bigg]^2  \bigg) \\
& = \  \E{v \sim \hyperposterior} \Bigg[\frac{1}{2} \bigg( 1- \E{\xbf \sim \mathcal{D}_{\mathcal{X}}}\bigg[ \E{h \sim \posterior_v} h(x^v) \bigg]^2  \bigg) \Bigg] 
\\
&
 =\ \E{v \sim \hyperposterior}  d_{\Dcal}(Q_v)  \,.
\end{align*}
By replacing $R_{\mathcal{D}}(G_{\rho})$ and $\dmv$ in  Equation~\eqref{eq:cbound_multiview_4}, we obtain
\begin{align*}
1 -  \frac{\displaystyle  \big(1-2R_{\Dcal}(G_{\hyperposterior})\big)^2}{\displaystyle 1-2 \dmv}
 \leq 
1 -  \frac{  \Big(1-2  \Esp_ {v \sim \hyperposterior} R_{\Dcal}(G_{Q_v})\Big)^2}{ 1-2 \Esp_ {v \sim \hyperposterior}  d_{\Dcal}(Q_v) } \,.
\end{align*}
\end{proof}

Equation~\eqref{eq:cbound_multiview_4} suggests that a good trade-off between the Gibbs risk and the disagreement between pairs of voters will lead to a well-performing majority vote.
Equation~\eqref{eq:cbound_multiview_2} controls the diversity among the views (important for multiview learning~\cite{Massih09, Goyal2017, chapelle2006semi, Maillard09}) thanks to the disagreement's expectation over the views $\Esp_ {v \sim \hyperposterior}  d_{\Dcal}(Q_v)$.

\subsection[Multiview PAC-Bayesian Theorem]{The General Multiview PAC-Bayesian Theorem}
\label{sec:theorems}
In this section, we give a general multiview PAC-Bayesian theorem~\cite{Goyal2017} that takes the form of a generalization bound for the Gibbs risk in the context of a two-level hierarchy of distributions. 
A key step in PAC-Bayesian proofs is the use of a \textit{change of measure inequality}~\citep{McAllester03}, based on the Donsker-Varadhan inequality~\citep{donsker1975}.
Lemma~\ref{lem:change} below extends this tool to our multiview setting.
\begin{lemma} 
	\label{lem:change}
	For any set of priors $\{\prior_v\}_{v=1}^V$ over $\{\Hcal_v\}_{v=1}^V$ and any set of posteriors $\{\posterior_v\}_{v=1}^V$ over $\{\Hcal_v\}_{v=1}^V$, for any hyper-prior distribution $\hyperprior$ on views $\Vcal$ and hyper-posterior distribution $\hyperposterior$ on $\Vcal$, 
	and for any measurable function $\phi  :  \mathcal{H}_v  \to  \mathbb{R}$, we have
	\begin{align*}
	& \E{ v \sim \hyperposterior}\  \E{ h \sim \posterior_v} \phi (h)  
	 \leq\  \E{ v \sim \hyperposterior}   \KL(\posterior_v \| \prior_v)  +  \KL(\hyperposterior \| \hyperprior)  +    \ln  \left(  \E{ v \sim \hyperprior}\   \E{ h \sim \prior_v}   e^{\phi (h)}  \right). 
	\end{align*}
\end{lemma}
\begin{proof}
Deferred to \ref{proof:change}
\end{proof}

Based on Lemma~\ref{lem:change}, the following theorem gives a generalization bound for multiview learning. 
Note that, as done by \cite{GermainLLM09,GermainLLMR15} we rely on a general convex function $D : [0,1]  \times  [0,1] \to \mathbb{R}$, which measures the ``deviation'' between the empirical and the true Gibbs risk.

\begin{theorem}
	\label{theo:MVPB1}
	Let $V \geq 2$ be the number of views. 
	For any distribution $\D$ on \mbox{$\Xcal\times\Ycal$}, for any set of prior distributions $\{\prior_v\}_{v=1}^V$ over $\{\Hcal_v\}_{v=1}^V$, for any hyper-prior distributions $\hyperprior$ over $\Vcal$, for any convex function $D : [0,1]\times[0,1]\to \mathbb{R}$, for any $\delta\in(0,1]$, 
	with a probability at least $1 - \delta$ over the random choice of $S \sim(\D)^n$, for all posterior $\{\posterior_v\}_{v=1}^V$  over $\{\Hcal_v\}_{v=1}^V$ and hyper-posterior $\hyperposterior$ over $\Vcal$ distributions, we have:
	\begin{align*}
	& D \left(   R_{S}(G_{\hyperposterior})
	, R_{\mathcal{D}}(G_{\hyperposterior})\right) \leq \frac{1}{m} \bigg[ \E{ v \sim \hyperposterior} \,  \KL(\posterior_v \| \prior_v)\\ 
	&\qquad  \qquad \qquad \qquad \qquad+  \KL(\hyperposterior \| \hyperprior) + \ln\left( \frac{1}{\delta}\, \E{ S \sim (\D)^n} \,  \E{ v \sim \hyperprior} \,  \E{ h \sim \prior_v}  e^{ n D\left(R_S(h) , R_\D(h)\right)} \right) \bigg] .
	\end{align*}
\end{theorem}
\begin{proof}
	First, note that $\E{ v \sim \hyperprior} \,  \E{ h \sim \prior_v} e^{ n D(R_S(h) , R_\D(h))}$ is a non-negative random variable. Using Markov's inequality, 
	with  $\delta  \in  (0,1]$, and a probability at least $1 - \delta$ over the random choice of the multiview learning sample $S \sim (\D)^n$, we have
	\begin{align*}
	\E{ v \sim \hyperprior} \  \E{ h \sim \prior_v} e^{ n\,D(R_S(h) , R_\D(h))} 
	\leq \ \frac{1}{\delta} \E{ S \sim (\D)^n} \  \E{ v \sim \hyperprior} \  \E{ h \sim \prior_v} e^{ n D(R_S(h) , R_\D(h))}.
	\end{align*}
	By taking the logarithm on both sides, with a probability at least $1 - \delta$ over  $S \sim (\D)^n$, we have
	\begin{align} 
	\label{eq:proof1}
	\ln \bigg [ \E{ v \sim \hyperprior} \  \E{ h \sim \prior_v} e^{ n\,D(R_S(h) , R_\D(h))} \bigg] 
	\leq   \ln \bigg[  \frac{1}{\delta} \E{ S \sim (\D)^n} \  \E{ v \sim \hyperprior} \  \E{ h \sim \prior_v} e^{ n D(R_S(h) , R_\D(h))} \bigg]
	\end{align}
	We now apply Lemma~\ref{lem:change} on the left-hand side of the Inequality~\eqref{eq:proof1} with \mbox{$\phi(h)= n\,D(R_S(h) , R_\D(h))$}. 
	Therefore,  for any $ \posterior_v$ on $\mathcal{H}_v$ for all views $v\in\Vcal$, and for any $\hyperposterior$ on views $\Vcal$, with a probability at least $1 - \delta$ over  $S \sim (\D)^n$, we have
	\begin{align*}
	& \ln \bigg[ \underset{ v \sim \hyperprior}{\mathbb{E}} \  \underset{ h \sim \prior_v}{\mathbb{E}} e^{ n\,D(R_S(h) , R_\D(h))} \bigg] \\ 
	& \ge\   n\, \underset{ v \sim \hyperposterior}{\mathbb{E}} \  \underset{ h \sim \posterior_v}{\mathbb{E}} D(R_S(h) , R_\D(h)) - \underset{ v \sim \hyperposterior}{\mathbb{E}} \  \KL(\posterior_v \| \prior_v) - \KL(\hyperposterior \| \hyperprior)   \\
	&\ge\   n\, D\left(\underset{ v \sim \hyperposterior}{\mathbb{E}} \  \underset{ h \sim \posterior_v}{\mathbb{E}} R_S(h) , \underset{ v \sim \hyperposterior}{\mathbb{E}} \  \underset{ h \sim \posterior_v}{\mathbb{E}} R_\D(h)\right) - \underset{ v \sim \hyperposterior}{\mathbb{E}} \ \KL(\posterior_v \| \prior_v) - \KL(\hyperposterior \| \hyperprior),
	\end{align*} 
	where the last inequality is obtained by applying Jensen's inequality on the convex function $D$. By rearranging the terms, we have
	\begin{align*}
	 D\left(\underset{ v \sim \hyperposterior}{\mathbb{E}} \  \underset{ h \sim \posterior_v}{\mathbb{E}} R_S(h) , \underset{ v \sim \hyperposterior}{\mathbb{E}} \  \underset{ h \sim \posterior_v}{\mathbb{E}} R_\D(h)\right)\,  & \leq \,   \frac{1}{m} \bigg[ \underset{ v \sim \hyperposterior}{\mathbb{E}} \  \KL(\posterior_v \| \prior_v) +  \KL(\hyperposterior \| \hyperprior) \\
	        +   \ln \bigg( \frac{1}{\delta} &\underset{ S \sim (\D)^n}{\mathbb{E}} \  \underset{ v \sim \hyperprior}{\mathbb{E}} \  \underset{ h \sim \prior_v}{\mathbb{E}} e^{ n\, D\left(R_S(h) , R_\D(h)\right)} \bigg)  	\bigg]
	\end{align*}
	Finally, the theorem statement is obtained by rewriting
	\begin{align}
	\Esp_{ v \sim \hyperposterior} \Esp_{ h \sim \posterior_v} R_S(h)  &= R_S(G_{\hyperposterior}), \label{eq:substitution1}\\ \label{eq:substitution2}
	\Esp_{ v \sim \hyperposterior} \Esp_{ h \sim \posterior_v} R_\D(h)  &= R_\D(G_{\hyperposterior})\,. 
	\end{align}
\end{proof}

Compared to the classical single-view PAC-Bayesian Bound of \cite{GermainLLM09,GermainLLMR15}, the main difference relies on the introduction of the view-specific prior and posterior distributions, which mainly leads to an additional term $\mathbf{E}_{ v \sim \hyperposterior} \KL(\posterior_v \| \prior_v)$ expressed as the expectation of the view-specific Kullback-Leibler divergence term over the views $\Vcal$ according to the hyper-posterior distribution $\hyperposterior$. 

Theorem~\ref{theo:MVPB1} provides tools to derive PAC-Bayesian generalization bounds for a multiview supervised learning setting. Indeed, by making use of the same trick as \cite{GermainLLM09,GermainLLMR15}, by choosing a suitable convex function $D$ and upper-bounding $\E{ S \sim (\D)^n} \  \E{ v \sim \hyperprior} \E{ h \sim \prior_v} e^{ n\, D(R_S(h) , R_\D(h))} $, we obtain instantiation of  Theorem~\ref{theo:MVPB1} .
In the next section we give an example of this kind of deviation through the approach of \cite{catoni2007pac}, that is one of the three classical PAC-Bayesian Theorems~\cite{catoni2007pac,McAllester99,Seeger02,Langford05}. 

\subsection[Instantiation of MV PB Theorem]{An Example of Instantiation of the Multiview PAC-Bayesian Theorem} 
\label{sec:catoni}
To obtain the following theorem which is a generalization bound with the \cite{catoni2007pac}'s point of view, we put $D$ as 
$D(a,b)=\mathcal{F}(b)-C\,a$ where $\cal F$ is a  convex function ${\cal F}$ and $C>0$ is a real number~\citep{GermainLLM09,GermainLLMR15}.
\begin{cor}
	\label{cor:catoni}
	Let $V \geq 2$ be the number of views. 
	For any distribution $\mathcal{D}$ on $\Xcal \times \Ycal$, for any set of prior distributions $\{\prior_v\}_{v=1}^V$ on $\{\Hcal\}_{v=1}^V$, 
	for any hyper-prior distributions $\hyperprior$ over $\Vcal$, 
	for any $\delta  \in  (0,1]$, with a probability at least $1 - \delta$ over the random choice of $S \sim  (\D)^n$ for all posterior $\{\posterior_v\}_{v=1}^V$ and hyper-posterior $\hyperposterior$ distributions, we have:
	\begin{align*}
	 &R_{\mathcal{D}}(G_{\hyperposterior})\\
	\le &\frac{1}{1\!-\!e^{-C}}\Bigg(\! 1-\! \exp  \bigg[-\!\!\bigg(C\, R_S(G_\rho)
	+ \frac{1}{n} \Big[ \E{ v \sim \hyperposterior}  \KL(\posterior_v \| \prior_v) +  \KL(\hyperposterior \| \hyperprior) + \ln \tfrac{1}{\delta} 	\Big] \bigg) \bigg]\Bigg).
	\end{align*}
\end{cor}
\begin{proof} Deferred to \ref{proof:cor_catoni}.
\end{proof}

This bound has the advantage of expressing  a trade-off between the empirical Gibbs risk and the Kullback-Leibler divergences.

\subsection{A Generalization Bound for the C-Bound}

From a practical standpoint, as pointed out before, controlling the multiview C-Bound of Equation~\eqref{eq:cbound_multiview_2} can be very useful for tackling multiview learning.
The next theorem is a generalization bound that justify the empirical minimization of the multiview C-bound (we use in our algorithm $\pbboost$ derived in Section~\ref{sec:pbboost}).
\begin{theorem}
	\label{theo:C-Bound-Multiview}
	Let $V \geq 2$ be the number of views. 
	For any distribution $\D$ on $\Xcal \times \Ycal$, for any set of prior distributions $\{\prior_v\}_{v=1}^V$, for any hyper-prior distributions $\hyperprior$ over views $\Vcal$, and for any convex function $D : [0,1]  \times  [0,1]  \to \mathbb{R}$, 
	with a probability at least $1 - \delta$ over the random choice of $S \sim  (D)^n$ for all posterior $\{\posterior_v\}_{v=1}^v$ and hyper-posterior $\hyperposterior$ distributions, we have:
	\begin{align*}
	R_{\mathcal{D}}(B_{\rho}) \le 1- \frac{\bigg(1-2 \E{v \sim \rho} \mathtt{sup} \  \big(\mathbf{r}_{Q_v,\mathcal{S}}^{\delta/2} \big)\bigg)^2}{1-2 \E{v \sim \rho}  \mathtt{inf } \   \mathbf{d}_{Q_v,\mathcal{S}}^{\delta/2}}\,, 
	\end{align*}
	where
	\begin{align}
	\mathbf{r}_{Q_v,\mathcal{S}}^{\delta/2} &= \bigg\{ r : \kl(R_S(Q_v)\|r) \leq \frac{1}{n} \bigg[ \KL(\posterior_v \| \prior_v)  + \ln \frac{4 \sqrt{m}}{\delta} \bigg]    \mbox{ and } r \leq \frac{1}{2} \bigg\} \label{eq:r_interval}\\
\mbox{and}\quad 	\mathbf{d}_{Q_v,\mathcal{S}}^{\delta/2} &= \bigg\{ d : \kl(d_{Q_v} \pwr S\|d) \leq \frac{1}{n} \bigg[ 2.\KL(\posterior_v \| \prior_v)  + \ln \frac{4 \sqrt{m}}{\delta} \bigg]  \bigg\} 
	\label{eq:d_interval}
	\end{align}
\end{theorem}
\begin{proof}
  Similarly to Equations $(23)$ and $(24)$ of \cite{GermainLLMR15}, we define the sets $\mathbf{r}_{Q_v,\mathcal{S}}^{\delta/2}$ (Equation~\eqref{eq:r_interval}) and $\mathbf{d}_{Q_v,\mathcal{S}}^{\delta/2}$ (Equation~\eqref{eq:d_interval}) for our setting. Finally, the bound is obtained (from Equation~\eqref{eq:cbound_multiview_2} of Lemma \ref{lem:mv-cbound}) by replacing the view-specific Gibbs  risk $R_{\D}(G_{Q_v})$ by its upper bound $\mathtt{sup} \  \mathbf{r}_{Q_v,\mathcal{S}}^{\delta/2}$ and  expected disagreement $d_{\D}(Q_v)$ by its lower bound $\mathtt{inf} \ \mathbf{d}_{Q_v,\mathcal{S}}^{\delta/2}$.
\end{proof}

\section{The $\pbboost$ algorithm}
\label{sec:pbboost}

\begin{algorithm}
	\caption{$\pbboost$}\label{alg:pbboost}
	\textbf{Input: } Training set $S = (\mathbf{x}_i, y_i), \dots , (\mathbf{x}_n, y_n)$, where $\mathbf{x}_i = (x\pwr 1, x\pwr 2, \dots , x\pwr V)$ and $y_i \in \{-1,1\}$. \\
	\text{\hspace{1.3cm}} For each view $v \in \mathcal{V}$, a view-specific hypothesis set $\mathcal{H}_ v$.\\
	\text{\hspace{1.3cm}} Number of iterations $T$.
	
	\begin{algorithmic}[1]
		\For{$\mathbf{x}_i \in S$}
		\State $\D_1(\mathbf{x}_i) \gets \frac{1}{n}$
		\EndFor		
		\State $ \forall v \in \Vcal$ $\rho_v \pwr 1 \gets \frac{1}{V}$
		\vspace{8pt}
		
		\For{$t = 1,\dots , T$}
		\State $\forall v \in \Vcal, \   h_v \pwr t  \gets \argmin_{h \in \Hcal_v} \E{(\mathbf{x}_i, y_i) \sim \Dcal_t}\left[\Ind{ [h(x_i \pwr v) \ne y_i ]}\right] $
		
		\State  Compute error: $\forall v \in \Vcal, \ \epsilon_v \pwr t \gets \E{(\mathbf{x}_i, y_i ) \sim \Dcal_t}\left[\Ind{[h_v \pwr t (x_i \pwr v) \ne y_i]} \right]$
		
		\State Compute voter weights (taking into account view specific information):
		\begin{align*}
		&  \forall v \in \Vcal, Q_v \pwr t  \gets \frac{1}{2} \bigg[ \ln \bigg( \frac{1- \epsilon_v \pwr t }{\epsilon_v \pwr t} \bigg) \bigg]
		\end{align*}
		
		\State \textbf{Optimize} the multiview C-Bound to learn weights over the views
		\vspace{-8pt}
		\begin{align*}
		\mathrm{max}_{\rho} & \qquad \frac{\bigg[1 - 2 \sum_{v=1}^{V} \rho_v \pwr t r_v \pwr t \bigg]^{2} }{1- 2 \sum_{v=1}^{V} \rho_v \pwr t d_v \pwr t} \\
		s.t. & \qquad \sum_{v=1}^{V} \rho_v \pwr t =1 , \quad \rho_v \pwr t \ge 0 \quad \forall v \in \{ 1, ..., V\} \\
		\end{align*}
		\vspace{-45pt}
		\begin{align*}
		\mbox{ where }
		& \forall v \in \Vcal, \ r_v \pwr t \gets \E{(\mathbf{x}_i, y_i ) \sim \Dcal_t} \  \E{h \sim \mathcal{H}_v}\left[\Ind{[h  (x_i \pwr v) \ne y_i ]}\right]   \\ 
		& \forall v \in \Vcal, \ d_v \pwr t \gets \E{(\mathbf{x}_i, y_i ) \sim \Dcal_t} \  \E{h,h'\sim \mathcal{H}_v} \left[\Ind{[h  (x_i \pwr v) \ne h' (x_i \pwr v) ]}\right]  
		\end{align*}
		\For{$\mathbf{x}_i \in S$}
		\State $\D_{t+1}(\mathbf{x}_i) \gets \frac{ \D_t(\mathbf{x}_i) \exp \big({-y_i \sum_{v=1}^{V} \rho_v \pwr t Q_v \pwr t h_v \pwr t (x_i \pwr v) }\big)}{ \sum_{j=1}^{n} \D_t(\mathbf{x}_j) \exp \big({-y_j \sum_{v=1}^{V} \rho_v \pwr t Q_v \pwr t h_v \pwr t (x_j \pwr v) }\big)}$
		\EndFor
		
		\EndFor
		\State \textbf{Return: } For each view $ v \in \Vcal, $ weights over view-specific voters and weights over views i.e. $\rho \pwr T$
	\end{algorithmic}
	
\end{algorithm}

Following our two-level hierarchical strategy (see Figure \ref{fig:MultiviewHierarchy}),  we aim at combining the view-specific voters (or views) leading to a well-performing multiview majority vote given by Equation \eqref{eq:mv_majority_vote}.
Boosting is  a well-known approach which  aims at combining a set of weak voters in order to build a more  efficient classifier than each of the view-specific classifiers alone.
Typically, boosting algorithms repeatedly  learn a ``weak" voter using a learning algorithm  with different probability distribution over the learning sample $S$. 
Finally, it combines all the weak voters in order to have one single strong classifier which performs better than the individual weak voters.
Therefore, we exploit boosting paradigm to derive a multiview learning algorithm $\pbboost$ (see Algorithm~\ref{alg:pbboost})  for our setting.  

For a given training set $S =\{ (\mathbf{x}_i, y_i), \dots , (\mathbf{x}_n, y_n)\}\in(\mathcal{X}\times \{-1,+1\})^n$ of size $n$; the proposed algorithm (Algorithm \ref{alg:pbboost}) maintains a distribution over the examples which is initialized as uniform. Then at each iteration, $V$ view-specific weak classifiers are learned according to the current distribution $\Dcal_t$ (Step 5), and their corresponding errors $\epsilon^t_v$ are estimated (Step 6).

Similarly to the Adaboost algorithm \cite{adaboost}, the weights of each view-specific classifier $(Q_v^t)_{1\le v\le V}$ are then computed with respect to these errors as
\begin{align*}
& \forall v \in \Vcal, Q_v \pwr t  \gets \frac{1}{2} \bigg[ \ln \bigg( \frac{1- \epsilon_v \pwr t }{\epsilon_v \pwr t} \bigg) \bigg] 
\end{align*}

To learn the weights $(\rho_v)_{1\le v\le V}$ over the views,  we optimize the multiview C-Bound, given by Equation \eqref{eq:cbound_multiview_2} of Lemma \ref{lem:mv-cbound} (Step 8 of algorithm), which in our case writes as a constraint minimization problem
		\begin{align*}
		\mathrm{max}_{\rho} & \qquad \frac{\bigg[1 - 2 \sum_{v=1}^{V} \rho_v \pwr t r_v \pwr t \bigg]^{2} }{1- 2 \sum_{v=1}^{V} \rho_v \pwr t d_v \pwr t} \\
		s.t. & \qquad \sum_{v=1}^{V} \rho_v \pwr t =1 , \quad \rho_v \pwr t \ge 0 \quad \forall v \in \{ 1, ..., V\} 
		\end{align*}
where, $r_v$ is the view-specific Gibbs risk and,  $d_v$ the expected disagreement over all view-specific voters defined as follows.
\begin{align}
\label{eq:expected_gibbs_error} &  \ r_v \pwr t =  \E{(\mathbf{x}_i, y_i ) \sim \Dcal_t}  \   \E{h \sim \mathcal{H}_v} \Ind{[h  (x_i \pwr v) \ne y_i]}, \\ 
\label{eq:expected_disagreement} & \ d_v \pwr t = \E{(\mathbf{x}_i, y_i ) \sim \Dcal_t} \  \E{h,h'\sim \mathcal{H}_v} \Ind{[h  (x_i \pwr v) \ne h' (x_i \pwr v)]}.
\end{align}
Intuitively, the multiview C-Bound tries to diversify the view-specific voters and views (Equation \eqref{eq:expected_disagreement}) while controlling the classification error of the view-specific classifiers (Equation \eqref{eq:expected_gibbs_error}). 
This allows us to control the accuracy and the diversity between the views which is an important ingredient in multiview learning~\cite{emvboost,Goyal2017, Peng11,Peng17,MorvantHA14}. 

In Section \ref{sec:results}, we empirically show that our algorithm minimizes the multiview C-Bound over the iterations of the algorithm (this is theoretically justified by the generalization bound of Theorem~\ref{theo:C-Bound-Multiview}).
Finally, we update the distribution over training examples $\mathbf{x}_i$ (Step 9), by following the Adaboost algorithm and in a way that the weights of misclassified (resp. well classified) examples by the final weighted majority classifier increase (resp. decrease).
\begin{align*}
\D_{t+1}(\mathbf{x}_i) \gets \frac{ \D_t(\mathbf{x}_i) \exp \big({-y_i \sum_{v=1}^{V} \rho_v \pwr t Q_v \pwr t h_v \pwr t (x_i \pwr v) }\big)}{ \sum_{j=1}^{n} \D_t(\mathbf{x}_j) \exp \big({-y_j \sum_{v=1}^{V} \rho_v \pwr t Q_v \pwr t h_v \pwr t (x_j \pwr v) }\big)}
\end{align*}
Intuitively, this forces the view-specific classifiers to be consistent with each other, which is important for multiview learning~\cite{janodet:hal-00403242,Koco11,Xiao12}.
Finally, after $T$ iterations of algorithm, we learn the weights over the view-specific voters and weights over the views leading to a well-performing weighted multiview majority vote
\begin{align*}
\label{eq:Multiview_MV}
B_{\rho}(\mathbf{x})= \sign \left(\sum_{v=1}^{V} \rho_v \pwr T \sum_{t=1}^{T}   Q_v \pwr t h_v \pwr t(x\pwr v) \right).
\end{align*}

\section{Experimental Results}
\label{sec:results}

In this section, we present experiments to show the potential of our algorithm $\pbboost$ on the following datasets. 

\subsection{Datasets}
\label{sec:datasets}
\subsubsection*{MNIST}
MNIST is a publicly available dataset consisting of $70,000$ images of handwritten digits distributed over ten classes~\cite{Lecun98}. 
For our experiments, we generated $2$ four-view datasets where each view is a vector of $\mathbb{R} \pwr {14 \times 14}$.
Similarly than done by ~\cite{ChenD17}, the first dataset ($\mathtt{MNIST}_1$) is generated by considering $4$ quarters of image as $4$ views.
For the second dataset ($\mathtt{MNIST}_2$) we consider $4$ overlapping views around the centre of images: this dataset brings redundancy between the views.
These two datasets allow us to check if our algorithm is able to capture  redundancy between the views.
We reserve $10,000$ of images as test samples and remaining as training samples. 

\subsubsection*{Multilingual, Multiview Text categorization}
This dataset is a multilingual text classification data extracted from Reuters RCV1/RCV2 corpus\footnote{\url{https://archive.ics.uci.edu/ml/datasets/Reuters+RCV1+RCV2+Multilingual,+Multiview+Text+Categorization+Test+collection}}.
It consists of more than $110,000$ documents written in five different languages (English, French, German, Italian and Spanish) distributed over six classes. 
We see different languages as different views of the data. 
We reserve $30 \%$ of documents as test samples and remaining as training data. 

\begin{table*}	
	\centering
	\resizebox{\textwidth}{!}{
		\begin{tabular}{ c| c c  c c c c c  c} 
			\hline
			\multirow{2}{*}{Strategy} &
			\multicolumn{2}{c}{\MNIST{1}} &&
			\multicolumn{2}{c}{\MNIST{2}} &&
			\multicolumn{2}{c}{\Reuters} \\
			\cline{2-3}\cline{5-6}\cline{8-9}  
			&  Accuracy & $F_1$  && Accuracy & $F_1$ && Accuracy & $F_1$  \\\hline
			$\mono$  & $.9034 \pm .001^{\downarrow}$ &  $.5353 \pm .006^{\downarrow}$  && $.9164 \pm .001^{\downarrow}$ &  $.5987 \pm .007^{\downarrow}$ && $.8420 \pm .002^{\downarrow}$ &  $.5051 \pm .007^{\downarrow}$\\
			$\concat$  & $.9224 \pm .002^{\downarrow}$ &  $.6168 \pm .011^{\downarrow}$  && $.9214 \pm .002^{\downarrow}$ &  $.6142 \pm .013^{\downarrow}$ && $.8431 \pm .004^{\downarrow}$ &  $.5088 \pm .012^{\downarrow}$\\
			
			\Fusion{dt}{}  & $.9320 \pm .001^{\downarrow}$ &  $.5451 \pm .019^{\downarrow}$  && $.9366 \pm .001^{\downarrow}$ &  $.5937 \pm .020^{\downarrow}$ && $.8587 \pm .003^{\downarrow}$ &  $.4128 \pm .017^{\downarrow}$\\
			
			$\mvmv$  & $.9402 \pm .001^{\downarrow}$ &  $.6321 \pm .009^{\downarrow}$  && $.9450 \pm .001^{\downarrow}$ &  $.6849 \pm .008^{\downarrow}$ && $.8780 \pm .002^{\downarrow}$ &  $.5443 \pm .012^{\downarrow}$\\
			
			$\rboost$ & $.9256 \pm .001^{\downarrow}$ &  $.5315 \pm .009^{\downarrow}$  && $.9545 \pm .0007$ &  $.7258 \pm .005^{\downarrow}$ && $.8853 \pm .002$ &  $.5718 \pm .011^{\downarrow}$\\
			
			$\mvAdaboost$ & $\textit{.9514} \pm .001$ &  $.6510 \pm .012^{\downarrow}$  && $\textit{.9641} \pm .0009$ &  $.7776 \pm .007^{\downarrow}$ && ${.8942} \pm .006$ &  $.5581 \pm .013^{\downarrow}$\\
			$\mvboost$ & $.9494 \pm .003^{\downarrow}$ &  $\textit{.7733} \pm .009^{\downarrow}$  && $.9555 \pm .002$ &  $\textit{.7910} \pm .006^{\downarrow}$ && $.8627 \pm .007^{\downarrow}$ &  ${.5789} \pm .012^{\downarrow}$\\

			\Fusion{Cq}{all} & $.9418 \pm .002^{\downarrow}$ &  $.6120 \pm .040^{\downarrow}$  && $.9548 \pm .003^{\downarrow}$ &  $.7217 \pm .041^{\downarrow}$ && $\textbf{.9001 }\pm .003$ &  $\textbf{.6279 }\pm .019$\\
			
			$\pbboost$   & $\textbf{.9661} \pm .0009$ &  $\textbf{.8066} \pm .005$  && $\textbf{.9674} \pm .0009$ &  $\textbf{.8166} \pm .006$ && $\textit{.8953} \pm .002$ &  $\textit{.5960} \pm .015^{\downarrow}$\\
			\hline
			
		\end{tabular}}
		\caption{Test classification accuracy and \mbox{$F_1$-score} of different approaches averaged over all the classes and over $20$ random sets of $n=500$ labeled examples per training set. Along each column, the best result is in bold, and second one in italic. $^{\downarrow}$ indicates that a result is statistically significantly worse than the best result, according to a Wilcoxon rank sum test with $p < 0.02$.
		}
		\label{tab:results}
	\end{table*}
	
\begin{figure*}[t!]%
	\centering
	\begin{tabular}{cc}
		\hspace{-10mm}\includegraphics[scale=.205]{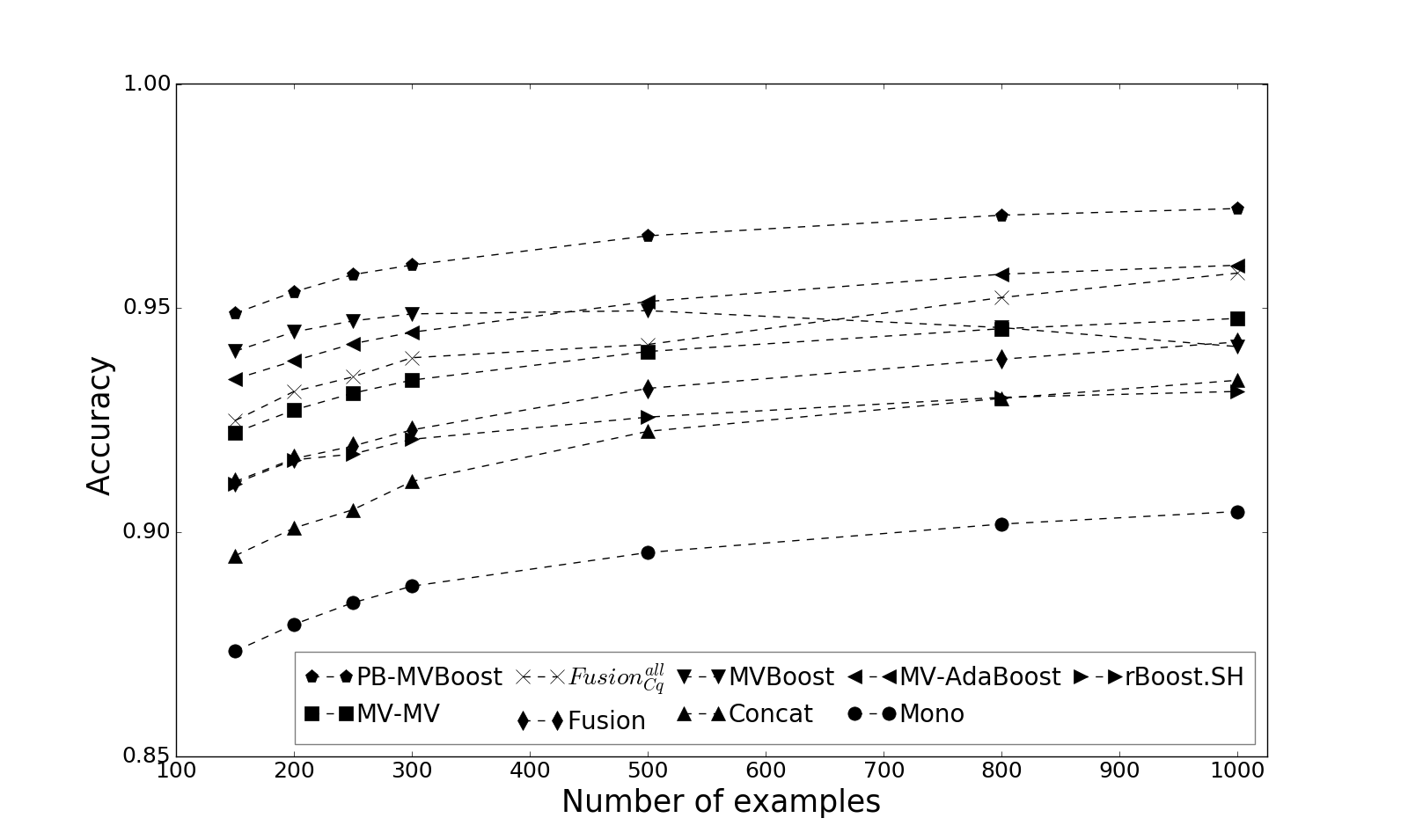}  & 
		\hspace{-1.0cm}\includegraphics[scale=.205]{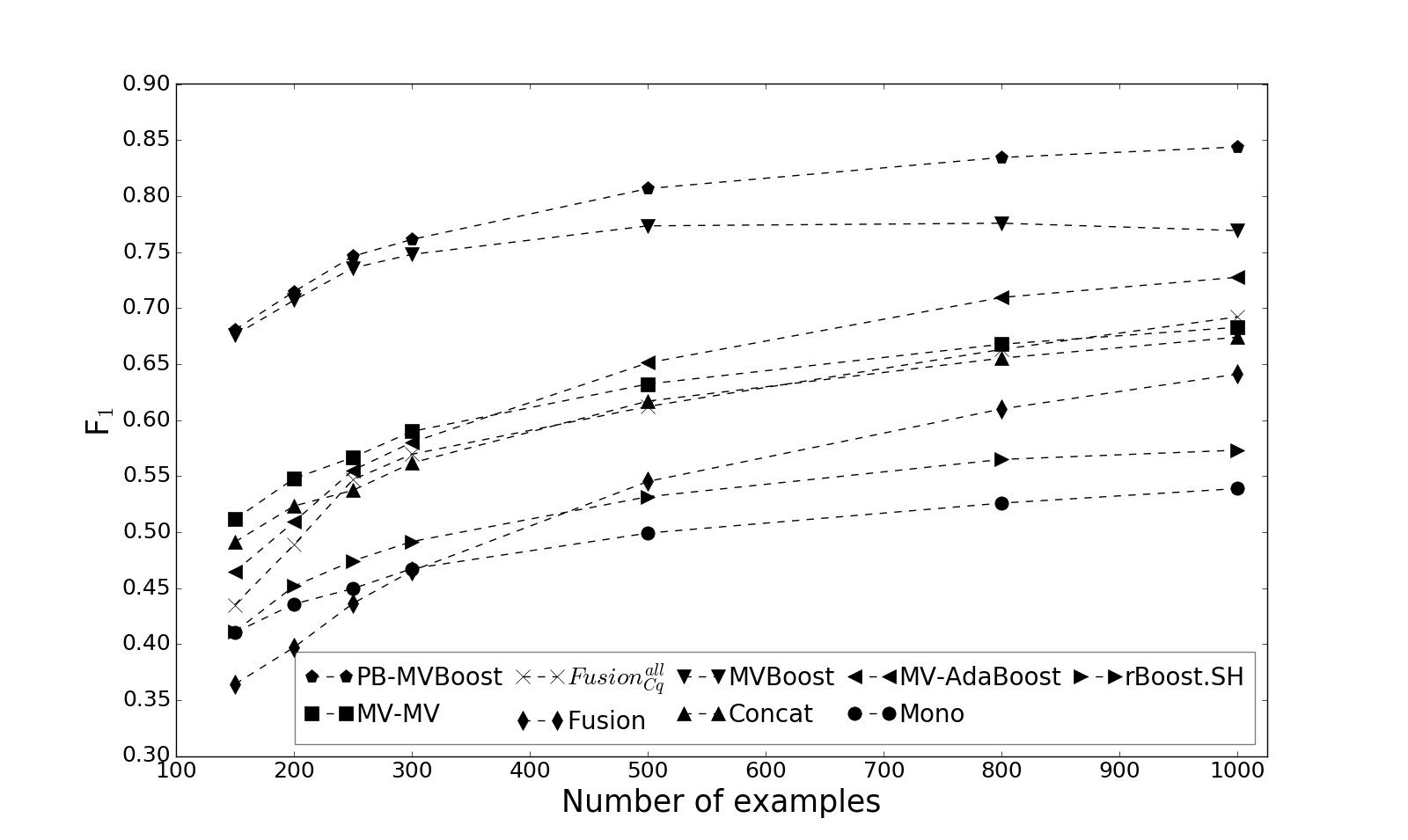}\\
		\multicolumn{2}{c}{(a) \MNIST{1}}\vspace{3mm}\\ 
		
		\hspace{-10mm}\includegraphics[scale=.205]{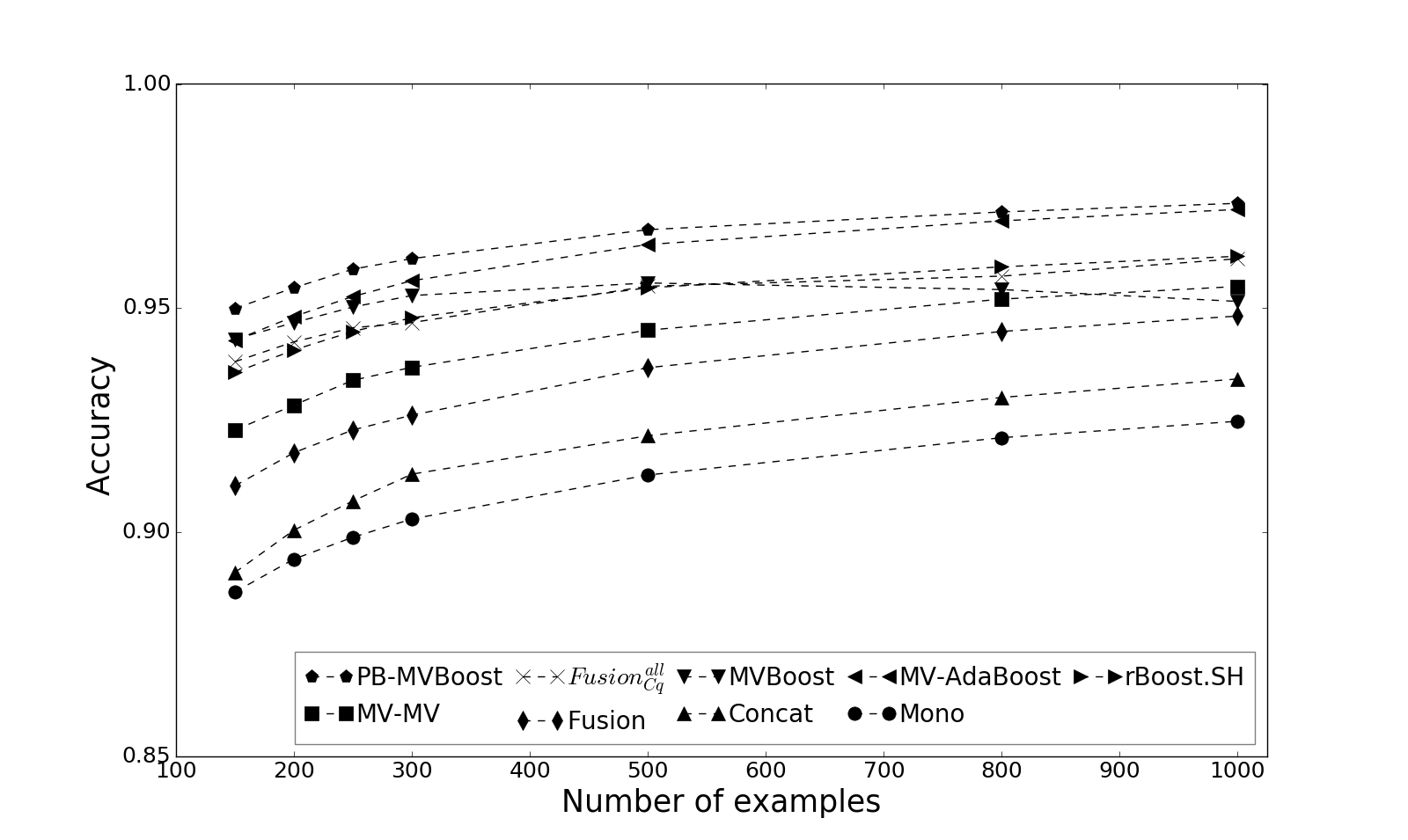}  & 
		\hspace{-1.0cm}\includegraphics[scale=.205]{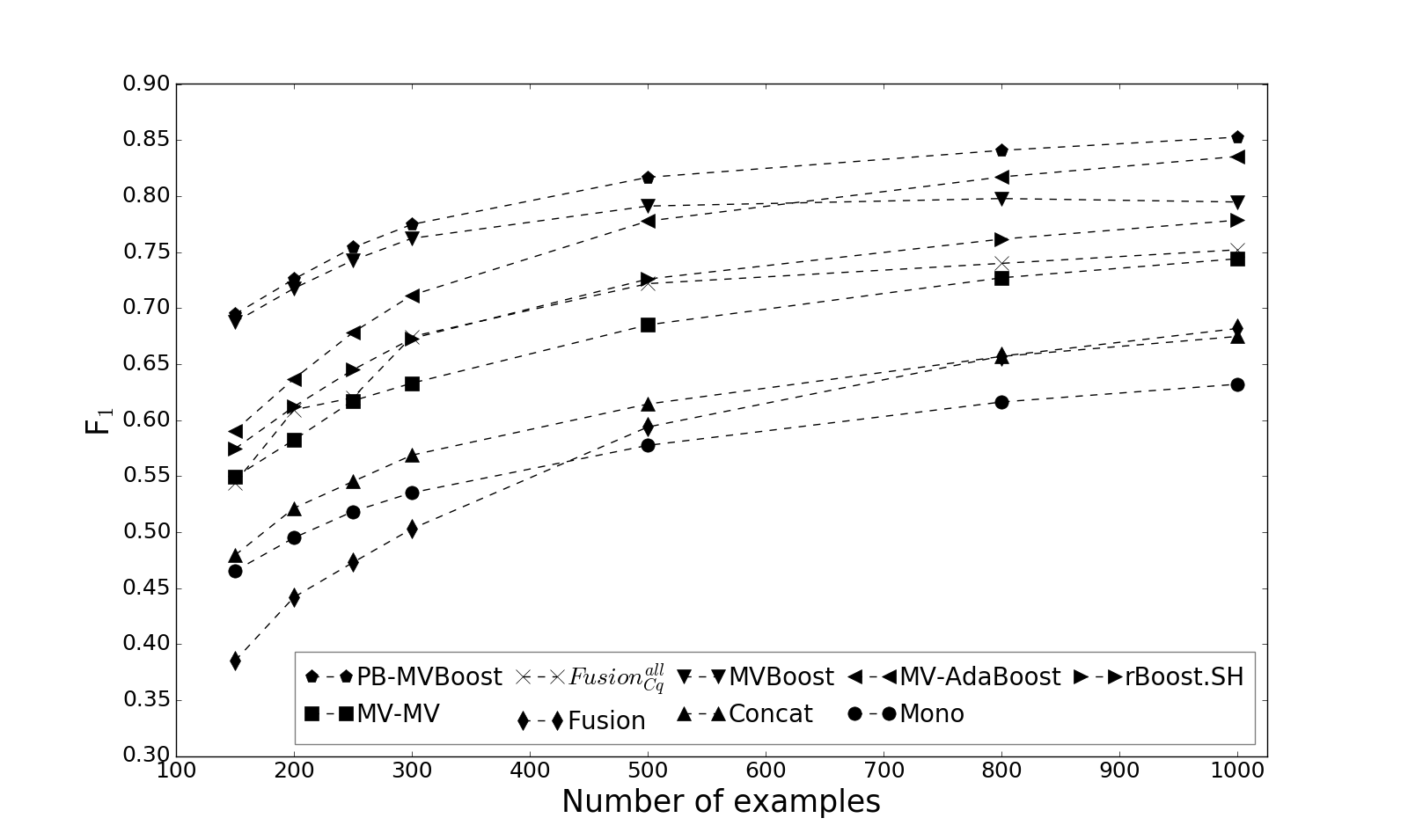} \\
		\multicolumn{2}{c}{(b) \MNIST{2}} \vspace{3mm}\\
		
		\hspace{-10mm}\includegraphics[scale=.205]{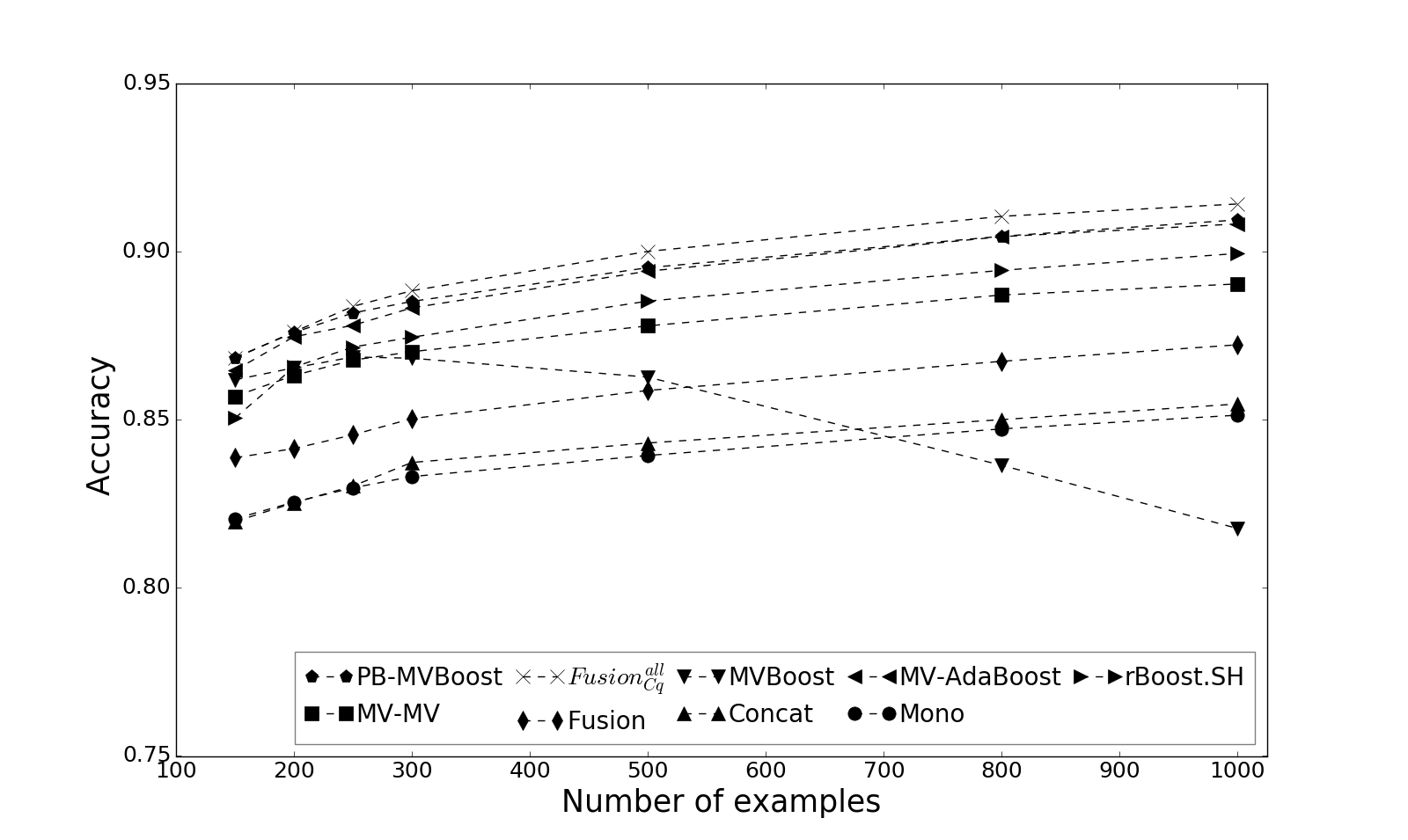}  &
		 \hspace{-1.0cm}\includegraphics[scale=.205]{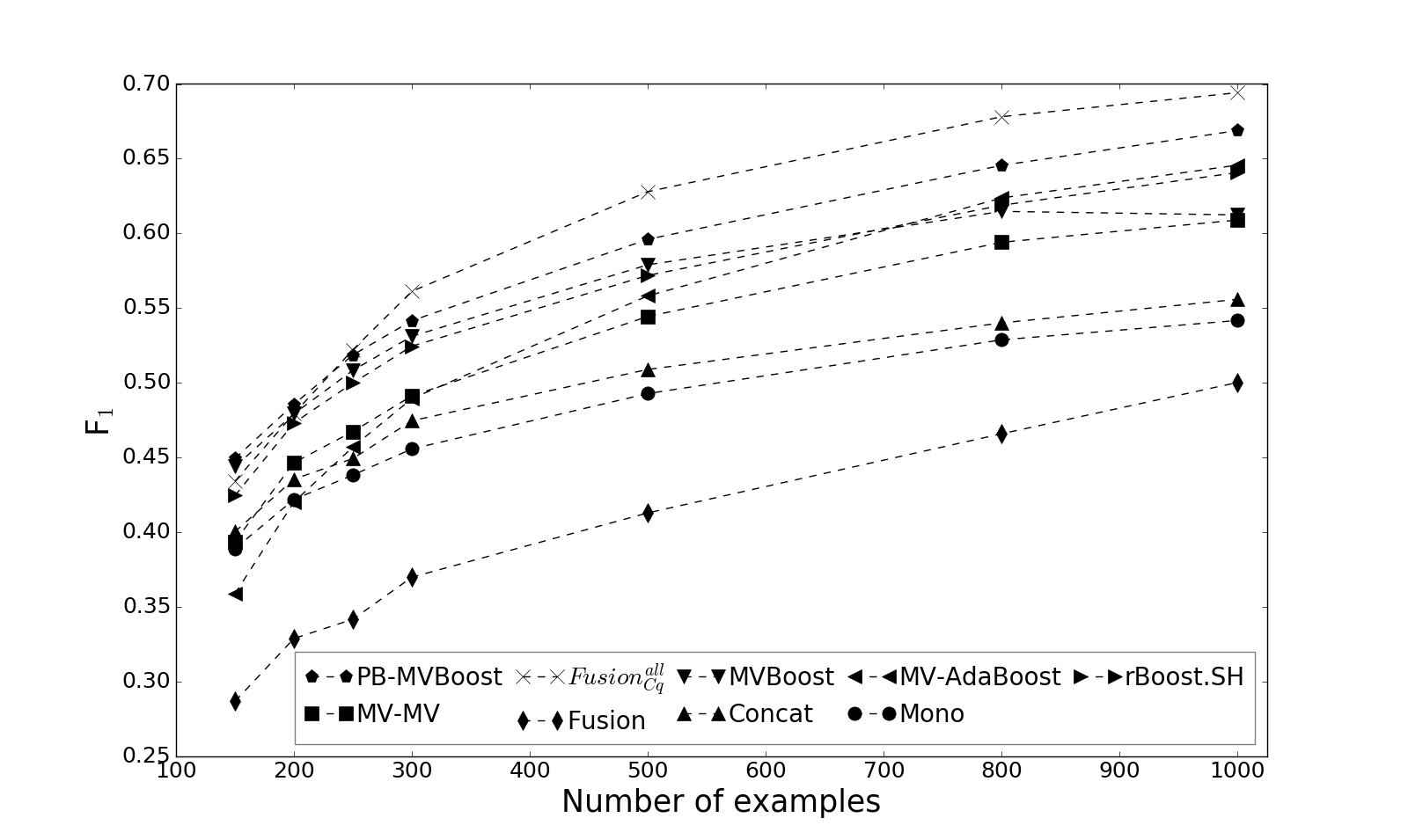} \\
		\multicolumn{2}{c}{(c) \Reuters}\\
	\end{tabular}
	\caption{Evolution of accuracy and $F_1$-measure with respect to the number of labeled examples in the initial labeled training sets on \MNIST{1}, \MNIST{2} and \Reuters{} datasets.}
	\label{fig:graphs}
\end{figure*}

\subsection{Experimental Protocol}
While the datasets are multiclass, we transformed them as  binary tasks by considering \textit{one-vs-all} classification problems: for each class we learn a binary classifier  by considering all the learning samples from that class as positive examples and the others as negative examples.
We consider different sizes of learning sample $S$ ($150$, $200$, $250$, $300$, $500$, $800$, $1000$) that are chosen randomly from the training data.
Moreover, all the results are averaged over $20$ random runs of the experiments.
Since the classes are unbalanced, we report the accuracy along with F1-measure for the methods and all the scores are averaged over all the {\it one-vs-all} classification problems.

We consider two multiview learning algorithms based on our two-step hierarchical strategy, and compare the $\pbboost$ algorithm described in Section~\ref{sec:pbboost}, with a previously developed multiview learning algorithm \cite{Goyal2017}, based on classifier late fusion approach \cite{Early-Late-ACMMultimedia05}, and referred to as \Fusion{Cq}{all}. 
Concretely, at the first level, this algorithm trains different view-specific linear SVM models with different hyperparameter $C$ values ($12$ values between $10 \pwr {-8}$ and $10 \pwr {3}$). 
And, at the second level, it learns a weighted combination over the predictions of view-specific voters using PAC-Bayesian algorithm $\texttt{CqBoost}$\cite{cqboost} with a RBF kernel. 
Note that, algorithm $\texttt{CqBoost}$ tends to minimize the PAC-Bayesian C-Bound \cite{GermainLLMR15} controlling the trade-off between accuracy and disagreement among voters.
The hyperparameter $\gamma$ of the RBF kernel is chosen over a set of $9$ values between $10 \pwr {-6}$ and $10 \pwr {2}$; and hyperparameter $\mu$ is chosen over a set of $8$ values between $10 \pwr {-8}$ and $10 \pwr {-1}$. 
To study the potential of our algorithms (\Fusion{Cq}{all} and $\pbboost$), we considered following $7$ baseline approaches:
  \begin{itemize}
  	\item $\mono $: We learn a view-specific model for each view using a decision tree classifier and report the results of the best performing view.
  	\item $\concat$: We learn one model using a decision tree classifier by concatenating features of all the views.
  	
  	\item \Fusion{dt}{}: This is a late fusion approach where we first learn the view-specific classifiers using $60 \%$ of learning sample. Then, we learn a final multiview weighted model over the  predictions of the view-specific classifiers.  For this approach, we used decision tree classifiers at both levels of learning.
  
  	\item $\mvmv$:  We compute a multiview uniform majority vote (similar to approach followed by \cite{Massih09}) over all the view-specific classifiers' outputs in order to make final prediction. We learn view-specific classifiers using decision tree classifiers.
  	
  	\item $\texttt{rBoost.SH}$: This is the multiview learning algorithm proposed by \cite{Peng11,Peng17} where a single global distribution is maintained over the learning sample for all the views and the  distribution over views are updated using multiarmed bandit framework. 
	  At each iteration, $\texttt{rBoost.SH}$ selects a view according to the current distribution and learns the corresponding view-specific voter.
          For tuning the parameters, we followed the same experimental setting as \cite{Peng17}.
  	 \item $\mvAdaboost$:  This is a majority vote classifier over the view-specific voters trained using Adaboost algorithm. 
  	 Here, our objective is to see the effect of maintaining separate distributions for all the views.
  	 
  	 \item $\mvboost$:  This is a variant of our algorithm $\pbboost$ but without learning weights over views by optimizing multiview C-Bound. Here, our objective is to see the effect of learning weights over views on multiview learning. 

  \end{itemize}
  
For all boosting based approaches ($\rboost$, $\mvAdaboost$, $\mvboost$ and  $\pbboost$), we learn the view-specific voters using a decision tree classifier with depth $2$ and $4$ as a weak classifier for \MNIST{}, and \Reuters{} RCV1/RCV2 datasets respectively. 
For all these approaches, we kept number of iterations $T=100$. 
For optimization of multiview C-Bound, we used Sequential Least SQuares Programming (SLSQP) implementation provided  by SciPy\footnote{\url{https://docs.scipy.org/doc/scipy/reference/optimize.minimize-slsqp.html}} \cite{SciPy}~and the decision trees implementation from scikit-learn~\cite{scikit-learn}.

\subsection{Results}

\begin{figure*}%
	\centering
	\includegraphics[scale=.3]{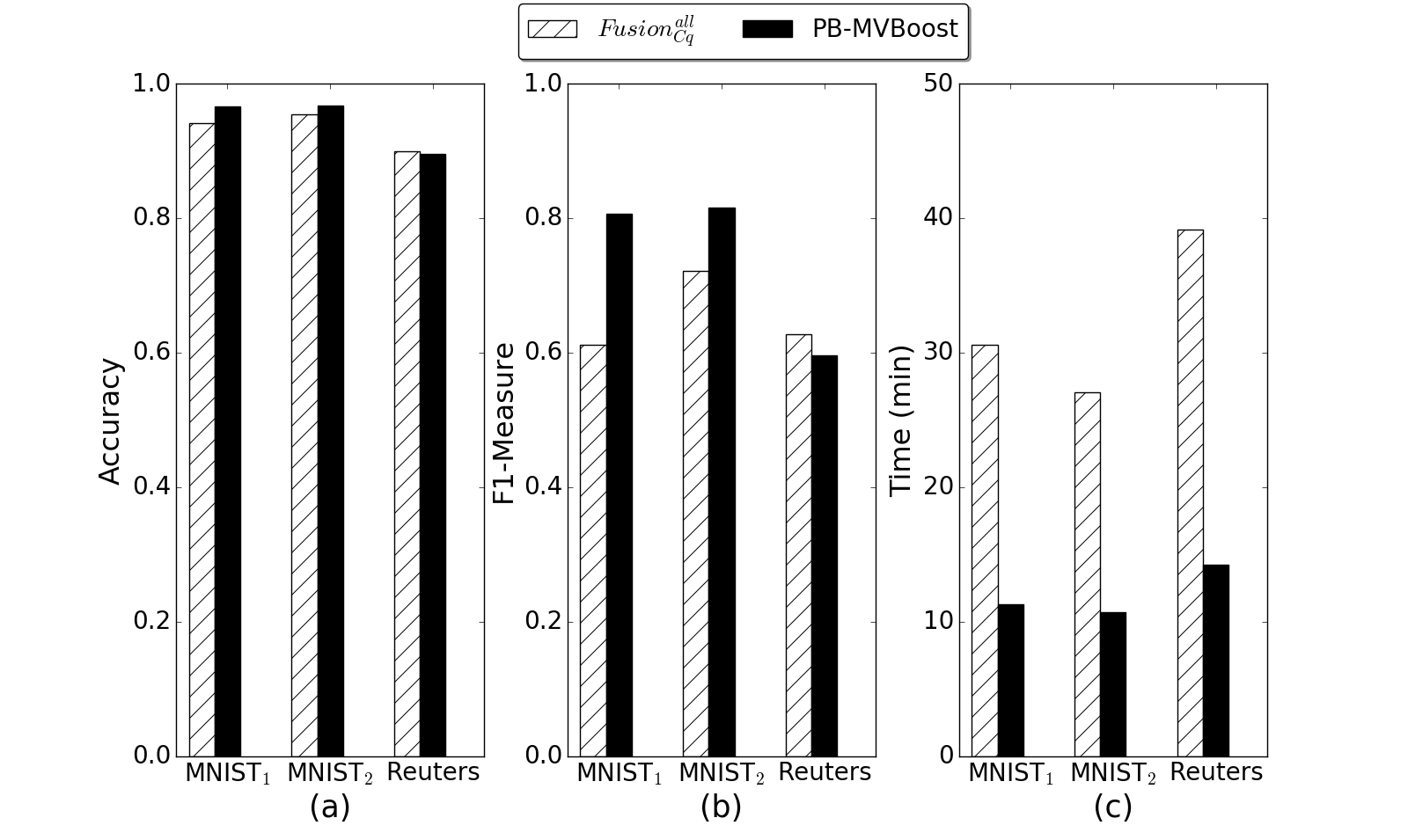}\\
	\caption{Comparison between \Fusion{Cq}{all} and $\pbboost$ in terms Accuracy (a), F1-Measure (b) and Time Complexity (c) for $n =500$ }
	\label{fig:histograms}
\end{figure*}

\begin{figure} 
	
	\begin{subfigure}{1.01\textwidth}
		\centering\includegraphics[scale=0.26]{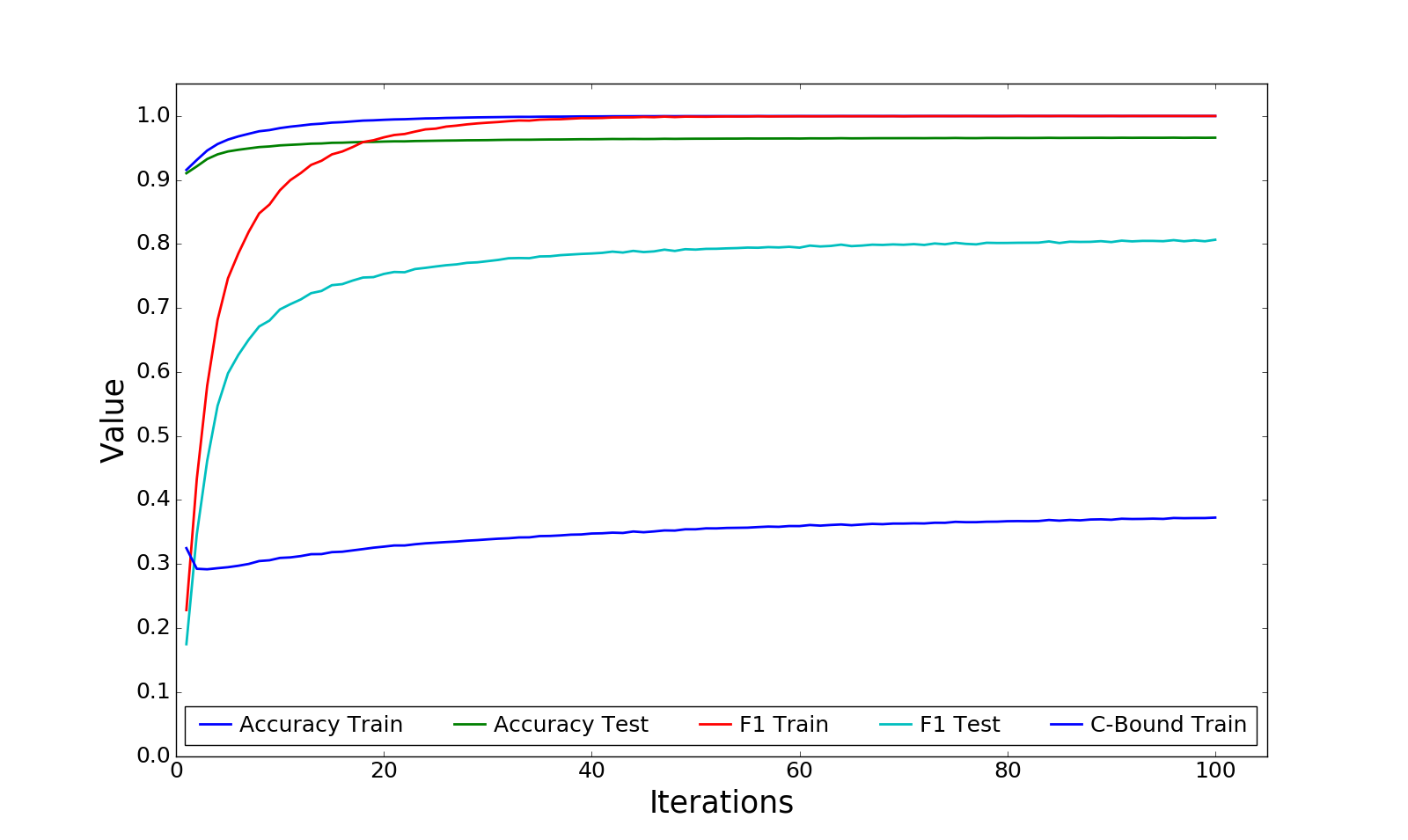}
		\caption{ \MNIST{1}} \label{fig:a}
	\end{subfigure}\hspace*{\fill}
	
	\begin{subfigure}{1.01\textwidth}
		\centering\includegraphics[scale=0.26]{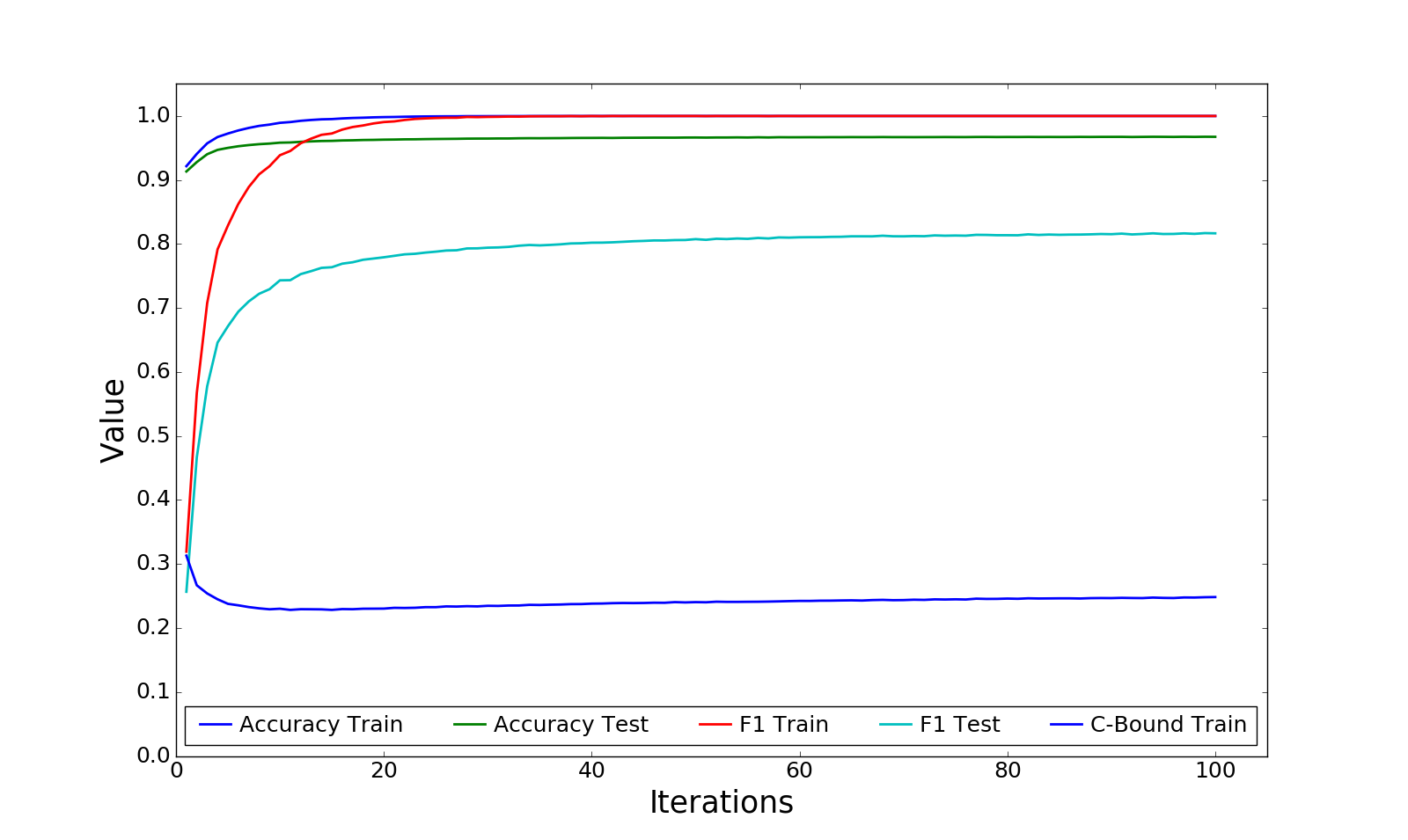}
		\caption{ \MNIST{2}} \label{fig:b}
	\end{subfigure}\hspace*{\fill}
	
	\begin{subfigure}{1.0\textwidth}
		\centering \includegraphics[scale=0.26]{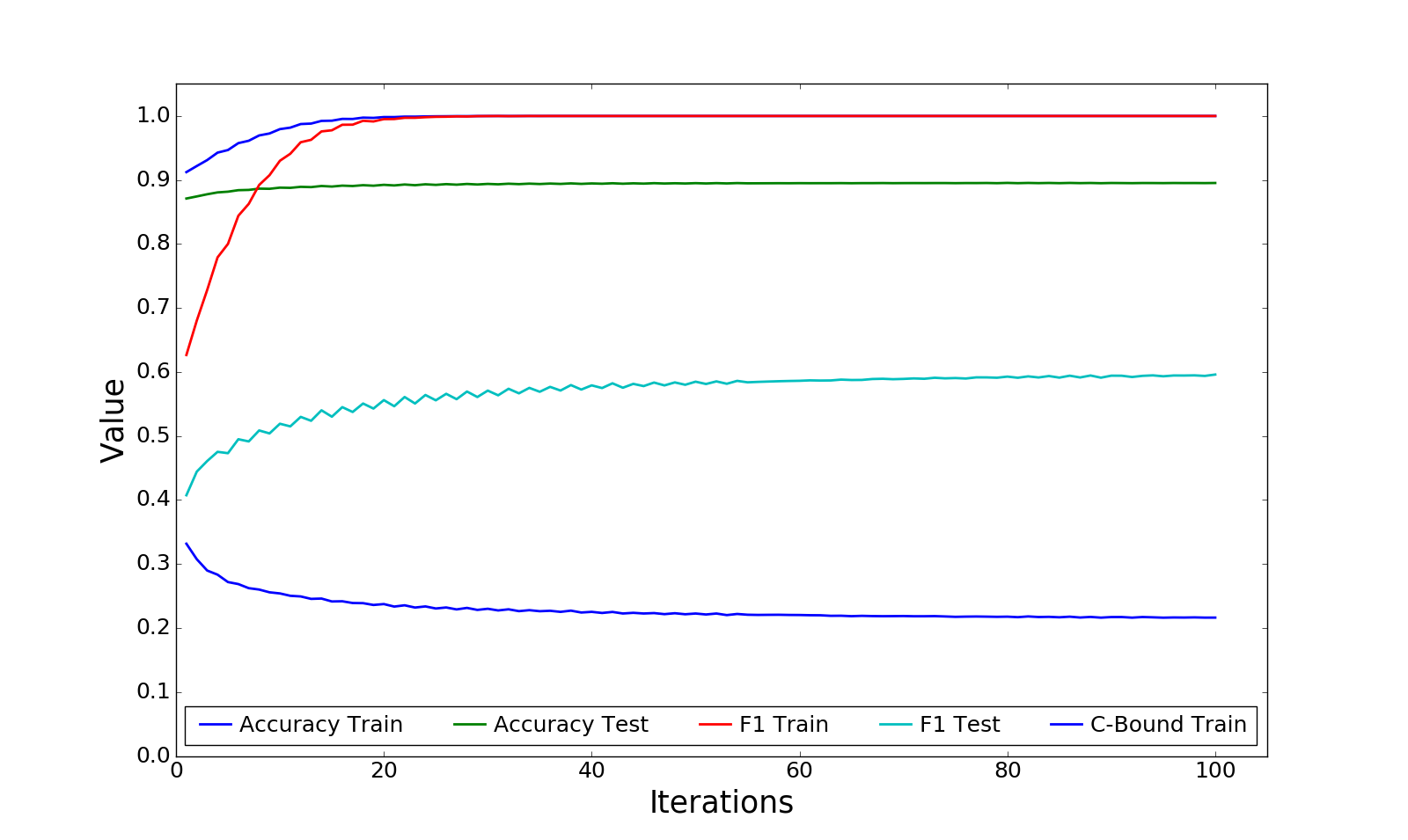}
		\caption{\Reuters{}} \label{fig:c}
	\end{subfigure}\hspace*{\fill}	
	\caption{Plots for classification error and F1-measure  on training and test data; and  empirical multiview C-Bound on training data over the iterations for all datasets with $n=500$.} \label{fig:plots}
\end{figure}

Firstly, we report the comparison of our algorithms \Fusion{Cq}{all} and $\pbboost$ (for $m=500$) with all the considered baseline methods in Table~\ref{tab:results}.
Secondly, Figure \ref{fig:graphs}, illustrates the evolution of the performances according to the size of the learning sample.
From the table, proposed two-step learning algorithm \Fusion{Cq}{all} is significantly better than the baseline approaches for \Reuters{} dataset.
Whereas, our boosting based algorithm $\pbboost$ is significantly better than all the baseline approaches for all the datasets.
This shows that considering a two-level hierarchical strategy in a PAC-Bayesian manner is an effective way to handle multiview learning.

In Figure  \ref{fig:histograms}, we compare proposed algorithms \Fusion{Cq}{all} and $\pbboost$ in terms of accuracy, $F_1$-score and time complexity for $m=500$ examples. 
For \MNIST{} datasets, $\pbboost$ is significantly better than \Fusion{Cq}{all}.
For \Reuters{} dataset,  \Fusion{Cq}{all} performs better than $\pbboost$  but computation time for \Fusion{Cq}{all} is much higher 
than that of $\pbboost$. 
Moreover, in Figure \ref{fig:graphs}, we can see that the performance (in terms of $F_1$-score) for  \Fusion{Cq}{all} is worse than $\pbboost$ when we have less training examples ($n=150 \mbox{ and } 200$).
This shows the proposed boosting based one-step algorithm $\pbboost$ is more stable and more effective for multiview learning.

From Table \ref{tab:results} and Figure \ref{fig:graphs}, we can observe that  $\mvAdaboost$ (where we have different distributions for each view over the learning sample) provides  better results compared to  other baselines in terms of accuracy but not in terms of F1-measure. 
On the other hand, $\mvboost$ (where we have single global distribution over the learning sample but without learning weights over views) is  better compared to other baselines in terms of F1-measure.
Moreover, the performances of $\mvboost$ first increases with an increase of the quantity of the training examples, then decreases.
Whereas our algorithm $\pbboost$ provides the best results in terms of both accuracy and F1-measure, and leads to a  monotonically increase of the performances with respect to the addition of labeled examples. 
This confirms that by maintaining a single global distribution over the views and learning the weights over the views using a PAC-Bayesian framework, we are able to take advantage of different representations (or views) of the data.

Finally, we plot behaviour of our algorithm $\pbboost$ over $T=100$ iterations on Figure~\ref{fig:plots} for all the datasets.
We plot accuracy and F1-measure of learned models on training and test data along with empirical multiview C-Bound on training data  at each iteration of our algorithm.
Over the iterations, the F1-measure on the test data keeps on increasing for all the datasets even if F1-measure  and accuracy on the training data reach the maximal value.
This confirms that our algorithm handles unbalanced data well. 
Moreover, the empirical multiview C-Bound (which controls the trade-off between accuracy and diversity between views) keeps on decreasing over the iterations.
This validates that by combining the PAC-Bayesian framework with the boosting one, we can empirically ensure the view specific information and diversity between the views for multiview learning.

\section{Conclusion}
\label{sec:conclusion}
In this paper, we provide a PAC-Bayesian analysis for a two-level hierarchical multiview learning approach with more than two views, when the model takes the form of a weighted majority vote over a set of functions/voters.
We consider a hierarchy of weights modelized by distributions where for each view we aim at learning \textit{i)}  posterior $Q_v$  distributions over the view-specific voters capturing the view-specific information and \textit{ii)}   hyper-posterior $\rho_v$  distributions over the set of the views.
Based on this strategy, we derived a general multiview  PAC-Bayesian theorem that can be specialized to any convex function to compare the empirical and true risks of the stochastic multiview Gibbs classifier.
We propose a boosting-based learning algorithm, called as $\pbboost$.
At each iteration of the algorithm,  we learn the weights over the view-specific voters and the weights over the views by optimizing an upper-bound over the risk of the majority vote (the multiview C-Bound) that has the advantage to allow to control a trade-off between accuracy and the diversity between the views.
The empirical evaluation shows that $\pbboost$ leads to good performances and confirms that our two-level PAC-Bayesian strategy is indeed a nice way to tackle multiview learning.
Moreover, we compare  the effect of maintaining separate distributions over the learning sample for each view; single global distribution over views; and single global distribution along with learning weights over views on results of multiview learning.
We show that by maintaining a single global distribution over the learning sample for all the views and learning the weights over the  views is an effective way to deal with multiview learning. 
In this way, we are able to capture the view-specific information and control the diversity between the views.
Finally, we compare $\pbboost$ with a two-step learning algorithm \Fusion{Cq}{all} which is based on PAC-Bayesian theory.
We show that $\pbboost$ is more stable and computationally faster than  \Fusion{Cq}{all}.

For future work, we would like to specialize our PAC-Bayesian generalization bounds to linear classifiers \cite{GermainLLM09} which will clearly open the door to derive theoretically founded multiview learning algorithms.
We would also like to extend our algorithm to \textit{semi-supervised} multiview learning where one has access to an additional unlabeled data during training. 
One possible way is to learn a view-specific voter using pseudo-labels (for unlabeled data) generated from the voters trained from other views (as done for example in \cite{Xu16}).
Another possible direction is to make use of unlabeled data while computing view-specific disagreement for optimizing multiview C-Bound. This clearly opens the door to derive theoretically founded algorithms for \textit{semi-supervised} multiview learning using PAC-Bayesian theory.
We would like to extend our algorithm to transfer learning setting where training and test data are drawn from different distributions.
An interesting direction would be to bind the data distribution to the different views of the data, as in some recent zero-shot learning approaches \cite{Socher13}.
Moreover, we would like to extend our work to the case of missing views or incomplete views e.g. \cite{Massih09}  and  \cite{xu15}.
One possible solution is to learn the view-specific voters using available view-specific training examples and adapt the distribution over the learning sample accordingly.

\section*{Acknowledgements}
This work was partially funded by the French ANR project LIVES ANR-15-CE23-0026-03 and the ``R\'egion  Rh\^{o}ne-Alpes''.

\appendix

\section*{Appendix}

\section{Mathematical Tools}
\label{sec:Appendix}

\begin{theorem}[Markov's ineq.]
	\label{theo:markov}
	For any random variable $X$ {\it s.t.} $\mathbb{E}(|X|)\!=\! \mu$, for any $a\!>\!0$, we  have $\displaystyle \mathbb{P}(|X| \ge a) \le \frac{\mu}{a}.$
\end{theorem}

\begin{theorem}[Jensen's ineq.]
\label{theo:jensen}
For any random variable $X$, for any concave function $g$, we have $\displaystyle g(\Esp [X]) \ \ge\  \Esp [g(X)].$
\end{theorem}

\begin{theorem}[Cantelli-Chebyshev ineq.]
\label{theo:chebyshev}
For any random variable $X$ {\it s.t.} $\mathbb{E}(X)\!=\! \mu$ and $\mathbf{Var}(X)=\sigma^2$, and for any $a\!>\!0$, we  have
$\mathbb{P}(X - \mu \ge a) \le \frac{\sigma^2}{\sigma^2 + a^2}.$
\end{theorem}


\section{Proof of $\mathcal{C}$-Bound for Multiview Learning (Lemma 1)}
\label{sec:proof_cbound}

In this section, we present the proof of Lemma~\ref{lem:mv-cbound}, inspired by the proof provided by \cite{GermainLLMR15}.
Firstly, we need to define the margin of the multiview weighted majority vote $B_\rho$ and its first and second statistical moments.
\begin{definition}
	\label{def:margin}
	Let $M_{\rho}$ is a random variable that outputs the margin of the multiview weighted  majority vote on the example $(\xbf,y)$ drawn from distribution $\mathcal{D}$, given by:
	$$M_{\rho} (\xbf,y) = \E{v \sim \hyperposterior} \  \E{h \sim \posterior_v} y\, h(x^v). $$
	The first and second statistical moments of the  margin are respectively given by
	\begin{equation}
	\label{eq:FirstMoment}
	\mu_{1}(M_{\rho}^{\mathcal{D}}) =  \underset{(\xbf,y) \sim \mathcal{D}}{\mathbf{E}} M_{\rho} (\xbf,y). 
	\end{equation}
	and, 
	\begin{align}
	\label{eq:SecondMoment}
	\nonumber \mu_2 (M_{\rho}^{\mathcal{D}}) & =  \underset{(\xbf,y) \sim \mathcal{D}}{\mathbf{E}} \big[ M_{\rho} (\xbf,y) \big]^2 \\
	&= \E{\xbf \sim \mathcal{D}_{X}}  \  y^2\,\Big[ \E{v \sim \hyperposterior} \ \E{h \sim \posterior_v} h(x^v) \Big]^2 =\E{\xbf \sim \mathcal{D}_{X}} \Big[ \E{v \sim \hyperposterior} \ \E{h \sim \posterior_v} h(x^v) \Big]^2.
	\end{align}
\end{definition}
According to this definition, the risk of the multiview weighted majority vote can be rewritten as follows:
\begin{align*}
R_{\mathcal{D}}(\bqmv) = \underset{(\xbf,y) \sim \mathcal{D}}{\mathbb{P}}  \big( M_{\rho} (\xbf,y) \le 0 \big).
\end{align*}

Moreover, the risk of the multiview Gibbs classifier can be expressed thanks to the first statistical moment of the margin.
Note that in the binary setting where $y\in\{-1,1\}$ and $h:\Xcal\to\{-1,1\}$, we have 
$\Ind{[h(x^v) \neq y]} = \frac12 (1-y\,h(x^v))$, and therefore
\begin{align}
R_{\mathcal{D}}(G_{\rho}) 
&  \nonumber 
= \E{(\mathbf{x},y) \sim \mathcal{D}} \ \E{ v \sim \hyperposterior} \ \E{h \sim \posterior_v} \Ind{[h(x^v) \neq y]} \\
& \label{eq:rrr}
= \frac{1}{2}\bigg(1- \E{(\xbf,y) \sim \mathcal{D}} \  \E{v \sim \hyperposterior} \ \E{h \sim \posterior_v} y\,h(x^v) \bigg)  \\
& \nonumber
= \frac{1}{2}(1-\mu_{1}(M_{\rho}^{\mathcal{D}}))\,.
\end{align}
Similarly, the expected disagreement can be expressed thanks to the second statistical moment of the margin by
\begin{align}
\dmv 
&  \nonumber
=  \Esp_{\xbf \sim \mathcal{D}_{\mathcal{X}}} \Esp_ {v \sim \hyperposterior} \Esp_{v' \sim \hyperposterior}   \Esp_{h \sim \posterior_v}  \Esp_{h' \sim \posterior_{v'}} \Ind{[ h(x^v) {\ne} h'(x^{v'})]}\\
&  \nonumber 
= \frac{1}{2} \bigg( 1- \Esp_{\xbf \sim \mathcal{D}_{\mathcal{X}}} \Esp_ {v \sim \hyperposterior} \Esp_{v' \sim \hyperposterior}   \Esp_{h \sim \posterior_v}  \Esp_{h \sim \posterior_{v'}}  h(x^v) \times h'(x^{v'})  \bigg) \\
&  \nonumber
= \frac{1}{2} \bigg( 1- \Esp_{\xbf \sim \mathcal{D}_{\mathcal{X}}} \Big[ \Esp_ {v \sim \hyperposterior}   \Esp_{h \sim \posterior_v}   h(x^v) \Big] \times \Big[ \Esp_{v' \sim \hyperposterior}  \Esp_{h' \sim \posterior_{v'}}  h'(x^{v'}) \Big] \bigg) \\
& \label{eq:ddd}
= \frac{1}{2} \bigg( 1- \E{\xbf \sim \mathcal{D}_{\mathcal{X}}}\bigg[  \E{v \sim \hyperposterior} \, \E{h \sim \posterior_v} h(x^v) \bigg]^2  \bigg) \\
& \nonumber
= \frac{1}{2}(1-\mu_{2}(M_{\rho}^{\mathcal{D}}))\,.
\end{align}
From above, we can easily deduce that $0 \le \dmv \le 1/2$ as $ 0 \le \mu_{2}(M_{\rho}^{\mathcal{D}}) \le 1$. Therefore, the variance of the margin can be written as:
\begin{equation}
\label{VarianceMargin}
\begin{split}
\text{Var}(M_{\rho}^{\mathcal{D}}) & =   \underset{(\xbf,y) \sim \mathcal{D}}{\textbf{Var}} (M_{\rho} (\xbf,y)) \\
& = \mu_2 (M_{\rho}^{\mathcal{D}}) - (\mu_1 (M_{\rho}^{\mathcal{D}}))^2.
\end{split}
\end{equation}

\subsection*{The proof of the  $\mathcal{C}$-bound} 

\begin{proof}
	By making use of one-sided Chebyshev inequality (Theorem \ref{theo:chebyshev} of \ref{sec:Appendix}), with $X=-M_{\rho} (\xbf,y)$, $\mu= \E{(\xbf,y) \sim \mathcal{D} } (M_{\rho} (\xbf,y))$ and $a= \E{(\xbf,y) \sim \mathcal{D}} M_{\rho} (\xbf,y)$, we have 
	\begin{align*}
	R_{\mathcal{D}}(B_{\hyperposterior})  & = \underset{(\xbf,y) \sim \mathcal{D}}{\mathbb{P}}  \big( M_{\rho} (\xbf,y) \le 0 \big) \\
	& = \underset{(\xbf,y) \sim \mathcal{D}}{\mathbb{P}} \bigg ( -M_{\rho} (\xbf,y) + \E{(\xbf,y) \sim \mathcal{D}} M_{\rho} (\xbf,y) \ge \E{(\xbf,y) \sim \mathcal{D}} M_{\rho} (\xbf,y) \bigg) \\
	& \le \frac{\underset{(\xbf,y) \sim \mathcal{D}}{\textbf{Var}} (M_{\rho} (\xbf,y))}{\underset{(\xbf,y) \sim \mathcal{D}}{\textbf{Var}} (M_{\rho} (\xbf,y)) + \bigg( \E{(\xbf,y) \sim \mathcal{D}} M_{\rho} (\xbf,y) \bigg)^2} \\
	& = \frac{\text{Var} (M_{\rho}^{\mathcal{D}})}{\mu_2(M_{\rho}^{\mathcal{D}}) - \bigg(\mu_1(M_{\rho}^{\mathcal{D}})\bigg)^2 + \bigg(\mu_1(M_{\rho}^{\mathcal{D}})\bigg)^2}\\
	& = \frac{\text{Var} (M_{\rho}^{\mathcal{D}})}{\mu_2(M_{\rho}^{\mathcal{D}})} \\
	& = \frac{\mu_2(M_{\rho}^{\mathcal{D}}) - \bigg(\mu_1(M_{\rho}^{\mathcal{D}})\bigg)^2}{\mu_2(M_{\rho}^{\mathcal{D}})} \\
	& = 1 - \frac{\bigg(\mu_1(M_{\rho}^{\mathcal{D}})\bigg)^2}{\mu_2(M_{\rho}^{\mathcal{D}})} \\
	& = 1 - \frac{\bigg(1-2\,R_{\mathcal{D}}(G_{\hyperposterior}) \bigg)^2}{1-2\,\dmv}
	\end{align*}
\end{proof}

\section{Proof of Lemma 2}
\label{proof:change}
 We have
	\begin{align*}
	\nonumber  \E{ v \sim \hyperposterior} \,  \E{ h \sim \posterior_v} \phi (h)\ &=\ \E{ v \sim \hyperposterior} \  \E{ h \sim \posterior_v} \ln e^{\phi (h)} \\
	\nonumber   &=\ \E{ v \sim \hyperposterior} \  \E{ h \sim \posterior_v} \ln \bigg( \frac{\posterior_v(h)}{\prior_v(h)} \frac{\prior_v(h)}{\posterior_v(h)} e^{\phi (h)} \bigg)  \\
	\nonumber &=\ \E{ v \sim \hyperposterior}\  \bigg[ \E{ h \sim \posterior_v} \ln \bigg( \frac{\posterior_v(h)}{\prior_v(h)} \bigg) +  \E{ h \sim \posterior_v} \ln \bigg( \frac{\prior_v(h)}{\posterior_v(h)} e^{\phi (h)} \bigg) \bigg].
	\end{align*}
	According to the Kullback-Leibler definition, we have
	\begin{align*}
	\nonumber \E{ v \sim \hyperposterior} \,  \E{ h \sim \posterior_v} \phi (h) \ =\   \E{ v \sim \hyperposterior} \bigg[  \KL(\posterior_v \| \prior_v) +     \E{ h \sim \posterior_v}   \ln  \bigg( 
	\frac{\prior_v(h)}{\posterior_v(h)} e^{\phi (h)}  \bigg)  
	\bigg].
	\end{align*}
	By applying Jensen's inequality (Theorem~\ref{theo:jensen}, in Appendix) on the concave function $\ln$, we have
	\begin{align*}
	\nonumber  \E{ v \sim \hyperposterior} \,  \E{ h \sim \posterior_v} \phi (h)  \ &\le \ \E{ v \sim \hyperposterior} \  \bigg[ \KL(\posterior_v \| \prior_v) +   \ln \bigg( \E{ h \sim \prior_v} e^{\phi (h)} \bigg) \bigg] \\
	\nonumber  &= \ \E{ v \sim \hyperposterior} \KL(\posterior_v \| \prior_v) + \E{ v \sim \hyperposterior} \ln \bigg( \frac{\hyperposterior(v)}{\hyperprior(v)} \frac{\hyperprior(v)}{\hyperposterior(v)} \E{ h \sim \prior_v} e^{\phi (h)} \bigg) \\
	\nonumber  &= \ \E{ v \sim \hyperposterior}  \KL(\posterior_v \| \prior_v) + \KL(\hyperposterior \| \hyperprior) 
	+    \E{ v \sim \hyperposterior}    \ln  \bigg( \frac{\hyperprior(v)}{\hyperposterior(v)}  \E{ h \sim \prior_v} e^{\phi (h)}  \bigg).
	\end{align*}
	Finally, we apply again the Jensen inequality (Theorem~\ref{theo:jensen}) on $\ln$ to obtain the lemma.

\section{A Catoni-Like Theorem---Proof of {Corollary 1}}
	\label{proof:cor_catoni}
The result comes from Theorem \ref{theo:MVPB1} by taking $D(a,b)\!=\!\mathcal{F}(b)\!-\!C a$, for a convex $\cal F$ and $C\!>\!0$, and by upper-bounding $\E{ S \sim (\D)^n}  \E{ v \sim \hyperprior} \E{ h \sim \prior_v} e^{ n D(R_S(h) , R_\D(h))}$. We consider $R_S(h)$ as a random variable following a binomial distribution of $n$ trials with a probability of success $R(h)$. We have:
	{
		\begin{align*}
		& \E{ S \sim (\D)^n} \,  \E{ v \sim \hyperprior}\, \E{ h \sim \prior_v} e^{ n\, D(R_S(h) , R_\D(h))} \\
		& = \!\E{ S \sim (\D)^n} \,  \E{ v \sim \hyperprior} \,\E{ h \sim \prior_v} e^{ n\, \mathcal{F} (R_\D(h)- C\,n\, R_S(h))} \\
		& {\small = \!\!\E{ S \sim (\D)^n}   \E{ v \sim \hyperprior} \E{ h \sim \prior_v} \!\!\!e^{ n\, \mathcal{F} (R_\D(h))} \sum_{k=0}^n \underset{S \sim (\D)^n}{\Pr}\!\!\left(\!R_S(h)\!=\!\frac{k}{n}\!\right)\!e^{\!-Ck}} \\
		& {\small  =\! \!\!\!\!\E{ S \sim (\D)^{\!n}}   \E{ v \sim \hyperprior}\, \E{ h \sim \prior_v}\!\!\! e^{ n \mathcal{F}\! (R_\D(h))\!} \sum_{k=0}^n\! {\textstyle \binom{n}{k}} R_\D(h)^k (1\!-\!R_\D(h))^{n\!-\!k\!} e^{-Ck} }\\
		&{\small  =\!\! \E{ S \sim (\D)^n} \,  \E{ v \sim \hyperprior} \,\E{ h \sim \prior_v}\! e^{ n \mathcal{F} (R_\D(h))} \big(R_\D(h)\,e^{-\!\,C}\!\!+\!(1\!-\!R_\D(h))\big)^n\!\!.}
		\end{align*}
	}%
	The corollary is obtained with $\mathcal{F}(p)\!=\!\ln \frac{1}{(1\!-\!p[1\!-\!e^{-C}])}$.


\bibliography{biblio}

\end{document}